\documentclass{article}

\usepackage{graphicx}
\usepackage{url}
\usepackage{amsmath,amsfonts,amssymb,amsthm}
\usepackage{algorithm}
\usepackage{algorithmic}

\newcommand{\N}{\mathbb{N}}
\newcommand{\R}{\mathbb{R}}
\newcommand{\Order}{\mathcal{O}}

\newcommand{\sign}{\operatorname{sign}}
\newcommand{\fp}[2]{\frac{\partial #1}{\partial #2}}
\newcommand{\TT}[1]{\mbox{\tiny $#1$}}
\newcommand{\Expectation}{\mathbb{E}}
\newcommand{\Normal}{\mathcal{N}}
\newcommand{\id}{\mathbf{I}}
\newcommand{\Projective}{\mathbb{P}}

\newcommand{\bmat}{\begin{pmatrix}}
\newcommand{\emat}{\end{pmatrix}}

\theoremstyle{plain}
\newtheorem{mathcounter}{Dummy}

\newtheorem{theorem}[mathcounter]{Theorem}
\newtheorem{lemma}[mathcounter]{Lemma}

\begin{document}

\title{Coordinate Descent with Online Adaptation\\of Coordinate Frequencies}

\author{Tobias Glasmachers\\
		Institut f\"ur Neuroinformatik\\
		Ruhr-Universit\"at Bochum, Germany\\[1em]
		{\"U}r\"un Dogan\\
		Institut f\"ur Mathematik\\
		Universit\"at Potsdam, Germany
}

\maketitle

\begin{abstract}
Coordinate descent (CD) algorithms have become the method of choice for
solving a number of optimization problems in machine learning. They are
particularly popular for training linear models, including linear
support vector machine classification, LASSO regression, and logistic
regression.
We consider general CD with non-uniform selection of coordinates.
Instead of fixing selection frequencies beforehand we propose an online
adaptation mechanism for this important parameter, called the adaptive
coordinate frequencies (ACF) method. This mechanism removes the need to
estimate optimal coordinate frequencies beforehand, and it automatically
reacts to changing requirements during an optimization run.
We demonstrate the usefulness of our ACF-CD approach for a variety
of optimization problems arising in machine learning contexts. Our
algorithm offers significant speed-ups over state-of-the-art training
methods.
\end{abstract}

\section{Introduction}
\label{sec:introduction}

Coordinate Descent (CD) algorithms are becoming increasingly important
for solving machine learning tasks. They have superseded other
gradient-based approaches such as stochastic gradient descent (SGD) for
solving certain types of problems, such as training of linear support
vector machines (SVMs) as well as LASSO regression and other $L_1$
regularized learning problems \cite{friedman:2007,hsieh:2008,liblinear}.
There is growing interest in machine learning applications of CD in the
field of optimizaation,
see e.g.~\cite{nesterov:2012,richtarik:2012}.

Natural competitors for solving large scale convex problems are
(trust region/pseudo) Newton methods and (stochastic) gradient descent.
In contrast to these approaches,
CD needs only a single component of the full gradient per iteration and
is thus particularly efficient if such a partial derivative is much
faster to compute than the gradient vector. This is often the case in
machine learning problems. The difference in computational effort per
step can be huge, differing by a factor as big as the number of data
points.
Stochastic gradient descent (SGD) has the very same advantage over
(plain) gradient descent. An ubiquitous problem of SGD is the need to
set a learning rate parameter, possibly equipped with a cooling schedule.
This is a cumbersome task, and success of a learning method---at least
within reasonable computational limits---can well depend on this choice.

CD algorithms do also have parameters. The most generic such parameter
is the frequency of choosing a coordinate for descent, e.g., in
randomized CD algorithms. This parameter is not obvious from a machine
learning perspective because uniform coordinate selection is apparently
dominant in all kinds of applications of CD. This is different in the
optimization literature on CD where non-uniform distributions have been
considered. This literature also offers a few criteria for choosing
selection probabilities (see e.g.\ \cite{nesterov:2012,richtarik:2013}
and refer to the more detailed discussion in section~\ref{section:CD}).
Interestingly these recommendations are all static in nature, i.e., the
selection probabilities are set before the start of the optimization run
and are then kept constant. This proceeding may be suitable for simple
(i.e., quadratic) objectives, however, it is difficult to propose good
settings of the parameters for realistic optimization scenarios.

In contrast, for SGD there have been a number of proposals for adapting
such parameters online during the optimization run for optimal progress
(see \cite{schaul:2013} and references therein). This does not only
effectively remove the need to adjust parameters to the problem instance
before the run, which is anyway often difficult due to missing
information. It also allows to react to changing requirements during the
optimization run. Similarly, trust region methods and many other
optimization strategies take online information into account for
adapting their parameters to the local characteristics of the problem
instance they are facing.

The present paper proposes an online adaptation technique for the
coordinate selection probability distribution of CD algorithms. We refer
to this technique as \textit{Adaptive Coordinate Frequencies} (ACF), and
to the resulting coordinate descent scheme as ACF-CD. We have first
proposed this algorithm in~\cite{glasmachers:2013arxiv,glasmachers:2013acml};
the present paper broadens and extends this works.

Our approach is inspired by previous work. First of all, the formulation
of our method is most natural in the context of random coordinate
descent as discussed by \cite{nesterov:2012}. It is closest in spirit to
the Adaptive Coordinate Descent algorithm by \cite{loshchilov:2011} that
adapts a coordinate system of descent directions online with the goal to
make steps independent. This algorithm maintains a number of state
variables (directions) that are subject to online adaptation. However,
this algorithm is deemed to be inefficient unless arbitrary directional
derivatives can be computed cheaply, which most often is not the case.
Online adaptation in general turns out to be a technique applied in many
different optimization strategies (refer to section~\ref{section:online}
for a detailed discussion).

The remainder of this paper is organized as follows. First we review the
basic coordinate descent algorithm with a focus on coordinate selection
techniques and summarize its use for the solution of various machine
learning problems. We then review online parameter adaptation techniques
applied in different (machine learning relevant) optimization methods.
Then we present our online parameter adaptation algorithm for coordinate
selection probabilities, followed by a Markov chain analysis of its
convergence behavior. The new algorithm is thoroughly evaluated on a
diverse set of problems against state-of-the-art CD solvers.

\section{Coordinate Descent}
\label{section:CD}

We consider convex optimization with variable $w = (w_1, \dots, w_n)$.
In the simplest case each $w_i \in \R$ is a real value, however, in
general we want to allow for a decomposition of the search space $\R^N$
into $n$ sub-spaces $w_i \in \R^{N_i}$ with $N = \sum_{i=1}^n N_i$. For
simplicity we refer to each component $w_i$ as a coordinate in the
following. The generalization to subspaces of more than one dimension is
implied. We denote the set of coordinate indices by
$I = \{1, \dots, n\}$.

Let $f : \R^N \to \R$ denote the objective function to be minimized.
Constraints may of course be present; they are handled implicitly in
this general presentation since the exact constraint handling technique
is problem specific.
The basic CD scheme for solving this problem iteratively is presented in
algorithm~\ref{algo:CD}.

\begin{algorithm}[h]
\begin{algorithmic}
	\STATE \textbf{input:} $w^{(0)} \in \R^n$
	\STATE $t \leftarrow 1$
	\REPEAT
		\STATE select active coordinate $i^{(t)} \in I$
		\STATE solve the optimization problem with additional constraints \\
		~~~~~~~~~~~~~~~~~~~~~~~~~~~~~~~~~~~~$w_j^{(t)} = w_j^{(t-1)}$ for all $j \in I \setminus \{ i^{(t)} \}$
		\STATE $t \leftarrow t + 1$
	\UNTIL{stopping criterion is met}
	\caption{Coordinate Descent (CD) algorithm.}
	\label{algo:CD}
\end{algorithmic}
\end{algorithm}

CD methods have advantages over other gradient-based optimization
schemes if the partial derivatives%
\footnote{We refer to $\fp{f}{w_i}(w)$ as a partial derivative of $f$.
  It is understood that in the case of subspace descent it consists of a
  vector of $N_i$ partial derivatives.}
$\fp{f}{w_i}(w)$ are significantly faster to compute than the complete
gradient~$\nabla_w f(w)$. In many machine learning problems this
difference is of order $\Theta(n)$, i.e., the computation of the full
gradient is about $n$~times more expensive than the computation of a
partial derivative. Based on a partial derivative a coordinate descent
solver performs a step on only the $i$-th coordinate by solving the
(often one-dimensional) sub-problem either optimally or approximately,
e.g., with a gradient descent step, a Newton step, line search, or with
a problem specific (possibly iterative) strategy.

Convergence properties of CD iterates and their values have been
established, e.g, in \cite{luo:1992,tseng:2001}, and runtime analysis
results in~\cite{richtarik:2013}.
A recent analysis in a machine learning context can be found
in~\cite{shalev:2013}.

\subsection{Coordinate Selection}

CD can come in a number of variations, e.g., differing in how the
sub-problem in each iteration~$t$ is solved. Here we want to highlight
the selection of coordinates.

A first difference is between deterministic and randomized
choices of coordinates $i^{(t)}$. The most prominent deterministic
scheme is the simple cyclic rule $i^{(t)} = t \bmod n$. It is often
implemented as an outer (epoch) loop and an inner loop sweeping over
all coordinates in the set~$I$. This means that all coordinates are
visited equally often and always in their natural order.
This basic scheme is sometimes randomized by permuting the indices in
the inner loop randomly in each epoch. The predefined but often
arbitrary order of the coordinates is thus avoided, but this method
still sticks to paying equal attention to each coordinate. Both
approaches can be viewed as uniform coordinate selection with different
dependency structures between variables.

Moving away from the epoch-based approach it is most natural to pick
$i^{(t)}$ i.i.d.\ at random from some distribution $\pi$ on $I$. We
denote the probability of selecting coordinate $i$ with $\pi_i$. The
simplest choice for $\pi$ is the uniform distribution. This choice seems
to be distinguished since it is the most obvious and in fact the only
unbiased one. Also, it allows to sample an index in constant time (based
on a random number generator that samples from the uniform distribution
on the unit interval).
Nesterov~\cite{nesterov:2012} proposes an algorithm for drawing a sample
from a non-uniform distribution $\pi$ in $\log(n)$ time, which is often
tolerable.

Better performance (in terms of progress per iteration) can be expected
if the best coordinate is chosen in each iterations, e.g., in a greedy
manner. This requires knowledge of the full gradient, a prerequisite
that usually renders CD methods inefficient. However, there are notable
exceptions. The standard solver for training non-linear (kernelized)
SVMs is based on the SMO algorithm \cite{platt:1998} but with highly
developed working set selection heuristics \cite{fan:2005,glasmachers:2006}.
When dropping the bias term from the SVM model these methods reduce to
CD with (approximately) greedy coordinate selection (refer to
\cite{steinwart:2011} for an extensive study). The reason for the
efficiency of CD in this case is that computation of the full gradient
takes $\Order(n^2)$ operations, but after a CD step the new gradient can
be obtained from the old one in only $\Order(n)$ operations. A similar
technique has been applied in \cite{hsieh:2011} for a sparse matrix
factorization problem. However, even linear time complexity is
prohibitive in many applications of CD algorithms. Then greedy selection
is not feasible and coordinate selection needs to revert to sampling
from a coordinate distribution~$\pi$.

\subsection{Non-uniform Distributions}

Despite the seemingly distinguished properties of the uniform
distribution it is in general implausible that selecting all coordinates
equally often should be optimal. In a machine learning problem a
coordinate often corresponds either to a training example or to a
feature, and it is understood that some data points (and some features)
are more important than others. Such important coordinates should be
chosen much more frequently than others.

However, knowing that a non-uniform distribution is advantageous does
not tell us in which direction to deviate from uniformity. The question
for the relative importance of coordinates for optimization is often
about as hard to answer as solving the optimization problem itself.

In the literature on CD the problem of finding good or even optimal
probabilities $\pi_i$ has been addressed mostly in terms of upper
runtime bounds, and under the additional assumptions that all partial
derivatives are Lipschitz continuous and that upper bounds on the
corresponding Lipschitz constants are known.
Nesterov derives a runtime bound (equation (2.12) in \cite{nesterov:2012})
that can be minimized given upper bounds on the Lipschitz constants, but
the conclusions drawn from his analysis instead consider the achievable
\emph{worst case} convergence rate as compared to other approaches.
A more direct approach is proposed by Richt{\'a}rik and Tak{\'a}{\v{c}}
in \cite{richtarik:2013} where minimization of a runtime bound for a
fixed problem instance is proposed as an optimization strategy (see
section~4 in~\cite{richtarik:2013}).

These approaches offer invaluable theoretical insights but turn out to
be problematic in practice. If coordinates correspond to data points or
features then the due to symmetrical treatment of all coordinates in the
machine learning optimization problem all a-priori upper bounds on the
Lipschitz constants of coordinate-wise derivatives coincide, resulting
in uniform coordinate selection. Data-dependent bounds can be tighter
and promise to be non-uniform. However, their computation may be too
costly to be practical. Even worse, for the procedure to be effective
these bounds would need to be continuously updated since the relative
importance of variables can change drastically during an optimization
run.

It may be for the reasons outlined above or just for simplicity of
concepts and implementations that standard algorithms in statistics and
machine learning rely nearly exclusively on uniform coordinate selection.
Actually, the only exception we are aware of is the shrinking technique
for linear SVM optimization (see section~\ref{subsection:linear-svm}).

\section{Coordinate Descent in Machine Learning}
\label{section:cd-in-ml}

CD methods have been popularized in statistics and machine learning
especially for certain regularized empirical risk minimization problems.
CD methods are particularly well suited for problems with sparse
solutions. One advantage is that they can quickly set single coordinates
to exact zero. This is in contrast to (stochastic) gradient descent,
which is often the most natural competitor. Hence intermediate solutions
are often sparse, which can greatly speed up computations. A follow up
advantage is that sparsity can be taken into account by coordinate
selection algorithms. This insight is at the heart of shrinking
techniques for SVM training~\cite{joachims:1998,fan:2005,liblinear}.

In machine learning, sparsity is often a result of regularization, most
prominently with an $L_1$ penalty on the weight vector of a linear model,
which is the case in least absolute shrinkage and selection operator
(LASSO) models \cite{friedman:2007}. Logistic regression with $L_1$
regularization is another prominent example \cite{yu:2011,yuan:2012a}.
Alternatively, sparsity (of the dual solution) can result from the
empirical risk term, e.g., in a support vector machine with hinge loss.
CD training of linear SVMs has been demonstrated to outperform competing
methods~\cite{hsieh:2008}.

In the following we present four prototypical supervised machine
learning problems that are commonly solved with CD algorithms. They will
serve as testbeds throughout this paper. Of course there exist many more
application areas such as sparse matrix factorization \cite{hsieh:2011},
stochastic variational inference \cite{hoffman:2013}, and others.
We start with data $\{(x_1, y_1), \dots, (x_\ell, y_\ell)\}$ composed of
inputs $x_i \in \R^d$ and labels $y_i \in Y$. Let $L(h(x), y)$ denote a
loss function comparing model outputs $h(x)$ and ground truth labels~$y$.
The (primal, unconstrained) regularized empirical risk minimization
training problem of the linear predictor $h_w(x) = \langle w, x \rangle$
amounts to
\begin{align}
	\min_{w \in \R^d} \quad f(w) = \frac{\lambda}{p} \big\|w\big\|_p^p + \frac{1}{\ell} \sum_{i=1}^\ell L\big(\langle w, x_i \rangle, y_i\big)
	\label{eq:trainingproblem}
\end{align}
where $p$ is typically either $1$ or $2$ and $\lambda > 0$ is a
complexity control parameter.

\subsection{The LASSO}

The LASSO problem is an instance of this problem with $p = 1$. In its
simplest form it is applied to a regression problem with $Y = \R$ and
$L(h(x), y) = \frac{1}{2} (h(x) - y)^2$.
Friedman et al.~\cite{friedman:2007} propose to solve this problem with
CD with a simple cyclic coordinate selection rule.

With all coordinates except $i$ fixed the resulting one-dimensional
problem is piecewise quadratic. The empirical risk term is a quadratic
term, and the regularizer restricted to coordinate $w_i$ reduces to
$\lambda |w_i|$. Given the partial derivative $\fp{f(w)}{w_i}$ this
problem can be solved in constant time: a gradient step equals a Newton
step since the second derivative equals one, and a case distinction
needs to be made for whether the component $w_i$ after the Newton step
is optimal, changed sign, or ends up at exact zero.

The most costly step is the computation of the derivative. It takes
$O(n_\text{nz})$ operations with $n_\text{nz}$ denoting the number of
non-zeros in the $i$-th \emph{column} of the data matrix $X$ composed
of the inputs $x_1, \dots, x_\ell$, i.e., the number of inputs with
non-zero $i$-th component~$(x_j)_i$. We will find in the following that
this situation is rather typical.

\subsection{Linear SVMs}
\label{subsection:linear-svm}

With binary classification labels $Y = \{-1, +1\}$, hinge loss
$L(h(x), y) = \max\{0,$ $1 - y h(x)\}$ and $p = 2$ we obtain the linear
soft margin SVM from equation~\eqref{eq:trainingproblem}. Hsieh et al.\
\cite{hsieh:2008} solve the corresponding dual problem
\begin{align}
	\min_{\alpha \in \R^\ell} \quad & f(\alpha) = \frac{1}{2} \sum_{i,j=1}^\ell \alpha_i \alpha_j y_i y_j \langle x_i, x_j \rangle - \sum_{i=1}^\ell \alpha_i \label{eq:SVM-dual} \\
	\text{s.t.} \quad & 0 \leq \alpha_i \leq C = \frac{1}{\lambda} \notag
\end{align}
with CD. The key technique for making CD iterations fast is to keep
track of the model vector $w = \sum_{i=1}^\ell \alpha_i y_i x_i \in \R^d$
during optimization. A CD step in this box-constrained quadratic program
amounts to a one-dimensional, interval-constrained Newton step. The
first derivative of the objective function w.r.t.\ $\alpha_i$ is
$y_i \langle w, x_i \rangle - 1$, the second derivative is
$\langle x_i, x_i \rangle$, which can be precomputed. The resulting
CD step reads
\begin{align*}
	\alpha^{(t)}_i = \left[ \alpha^{(t-1)}_i - \frac{1 - y_i \langle w, x_i \rangle}{\langle x_i, x_i \rangle} \right]_0^C
	\enspace,
\end{align*}
where $[x]_a^b = \max \big\{ a, \min\{b, x\} \big\}$ denotes truncation
of the argument $x$ to the interval $[a, b]$. With densely represented
$w \in \R^d$ and sparse data $x_i$ the complexity of a step is not
only independent of the data set size $\ell$, but even as low as the
number of non-zero entries in $x_i$ (and therefore often much lower
than the data dimension~$d$). Again we arrive at a complexity of
$\Order(n_\text{nz})$, where in this case $n_\text{nz}$ denotes the
number of non-zeros in the $i$-th \emph{row} of the data matrix~$X$.

The liblinear algorithm \cite{hsieh:2008,liblinear} applies a shrinking
heuristic that removes bounded variables from the problem during
optimization. This technique was originally proposed for non-linear SVM
optimization~\cite{joachims:1998}, and it can give considerable
speed-ups. From a CD perspective this technique sets the
probabilities~$\pi_i$ of removed coordinates to zero while normalizing
the remaining probabilities to a uniform distribution on the active
variables. As such it is the only technique in common use that actively
adapts coordinate selection probabilities~$\pi_i$ online during the
optimization run. It should be noted that the decision which coordinates
to remove is based on a heuristic. No matter how robust this heuristic
is designed, it is subject to infrequent failure, resulting in costly
convergence to a sub-optimal point, followed by a warm-start.

Online adaptation of $\pi$ is well justified for SVM training. This is
because the relative importance of coordinates changes significantly
over the course of an optimization run. Consider a variable $\alpha_i$
corresponding to an outlier $(x_i, y_i)$. At first, starting at
$\alpha_i = 0$, this coordinate is extremely important since it needs to
move by the maximal possible amount of~$C$. Once it arrives at the upper
bound the constraint $\alpha_i \leq C$ becomes active, essentially
fixing the variable at its current value. Thus its importance for
further optimization drops to zero. Of course the constraint may become
inactive later on. Hence, any a-priori estimation of relative importance
of coordinates (e.g., based on upper bounds on Lipschitz constants) is
not helpful in this case and online adaptation of $\pi$ is of uttermost
importance.

\subsection{Multi-class SVMs}

A non-trivial extension to this problem is multi-class SVM
classification. Different extensions to the binary problem exist. The
arguably simplest and most generic one is the one-versus-all approach
that reduces the multi-class problem to a set of two-class problems.
Other approaches attempt to generalize the margin concept to multiple
classes. There is the classic approach by Weston and Watkins
\cite{weston:1999} that turns out to be equivalent to a similar
proposals in \cite{vapnik:1998} and \cite{bredensteiner:1999}. A popular
alternative was proposed by Crammer and Singer~\cite{crammer:2002}, and
a considerably different approach by Lee et al.~\cite{lee:2004}.
All of these approaches differ only in the type of (piecewise linear)
large-margin loss function as a  multi-class replacement for the hinge
loss.

Although these methods were originally designed for non-linear SVMs they
can be applied for linear SVM training. From an optimization perspective
a decisive difference from binary classification is that the
corresponding dual problems contain $\Order(K)$ variables per training
example, where $K$ denotes the number of classes. This calls for a
solution with proper subspace descent with $N_i \in \Order(K) > 1$.
Sub-problems can be solved either with a general purpose QP solver or
with a SMO-style technique~\cite{platt:1998}.

We consider the ``default'' multi-class SVM proposed by Weston and
Watkins. Its dual problem is a box-constrained quadratic program that
could well be solved with a standard CD approach. Treating it as a
subspace descent problem is computationally attractive: once the partial
derivative is computed the (non-trivial) sub-problem can be solved to
arbitrary precision -- without a need for further derivative
computations in each sub-step.

\subsection{Logistic Regression}

Logistic regression is closely related to the linear SVM problem, but
with smooth loss function $L(h(x), y) = \log(1 + \exp(-y h(x)))$
replacing the non-smooth hinge loss. Its dual problem
\begin{align}
	\min_{\alpha \in \R^\ell} \quad & f(\alpha) = \frac{1}{2} \sum_{i,j=1}^\ell \alpha_i \alpha_j y_i y_j \langle x_i, x_j \rangle + \sum_{i=1}^\ell \alpha_i \log(\alpha_i) + (C - \alpha_i) \log(C - \alpha_i) \notag \\
	\text{s.t.} \quad & 0 \leq \alpha_i \leq C = \frac{1}{\lambda} \label{eq:logreg-dual}
\end{align}
can be solved efficiently with CD methods. The dual variables are
connected through a quadratic term, while the more difficult to handle
logarithmic terms as well as the constraints are separable. This problem
shares many properties with the dual linear SVM problem, however, the
logarithmic terms do not allow for an exact solution of the
one-dimensional CD sub-problem. Instead an iterative solver with second
order steps is employed in \cite{yu:2011} and implemented in
liblinear~\cite{liblinear}.
Also, the solution is dense which means that shrinking techniques are
not applicable. Thus this problem is solved with uniform coordinate
probabilities (coordinate sweeps in random order).

\section{Online Parameter Adaptation}
\label{section:online}

The coordinate sampling distribution $\pi$ is represented as an
$n$-dimensional vector $\pi \in \R^n$, subject to simplex constraints
$\pi_i \geq 0$ and $\sum_{i=1}^n \pi_i = 1$. This vector is a parameter
of the CD algorithm, and as such it can be subject to online tuning.
Before we propose such a procedure in the following section we review
existing parameter adaptation techniques.

Online adaptation of the values of algorithm parameters is found in many
different types of algorithms. Here we restrict ourselves to
optimization techniques. Driving this to the extreme one may consider
any non-trivial aspect of the state of an optimization algorithm that
goes beyond the current search point (and its properties, such as
derivatives) as parameters that may or may not be subject to online
adaptation.

Sometimes parameters can be set to robust default values, making online
adaptation essentially superfluous. However, it turns out that in many
cases performance can be increased when tuning such parameters to the
problem \emph{instance} at hand, which is often most efficiently done in
an online fashion. In the most extreme case an algorithm can break down
completely if online adaptation of a parameter is switched off. The
types of parameters present in different optimization algorithms differ
vastly. We revisit a few prototypical examples in the following.

\subsection{Stochastic Gradient Descent}

Stochastic Gradient Descent (SGD) algorithms are widespread in machine
learning. They are appreciated for their ability to deliver a usable
(although often poor) model even before the first sweep over the data
set is finished. In a machine learning context stochasticity is
introduced artificially into the gradient descent procedure by
approximating the empirical risk term
\begin{align*}
	\frac{1}{\ell} \sum_{i=1}^\ell L(h(x_i), y_i)
\end{align*}
with the loss $L(h(x_i), y_i)$ for a single pattern. This estimate is
unbiased, and this property carries over to its gradient
$\nabla_w L(h_w(x_i), y_i)$. For the sake of simplicity we assume
minimization of the empirical risk in the following. Then an SGD
iteration performs the update step
\begin{align*}
	w^{(t)} = w^{(t-1)} - \eta^{(t)} \cdot \nabla_w L(h_w(x_i), y_i)
\end{align*}
where $\eta^{(t)} > 0$ is a learning rate. Cooling schedules such as
$\eta^{(t)} \sim 1/t$ result in strong convergence guarantees
\cite{bottou:1998}, however, in some cases they turn out to be
inefficient in practice. This problem has been solved by online
adaptation, e.g., by Schaul et al.~\cite{schaul:2013}. In their approach
the objective function is modeled as a quadratic function. The model is
estimated online from the available information. The learning rate is
then adjusted in a way that is optimal given the current model. This
scheme was demonstrated to outperform plain SGD as well as number of
alternative methods on the task of training deep neural networks
(see \cite{schaul:2013} and references therein).

\subsection{Resilient Propagation}

The Resilient Propagation (Rprop) algorithm was originally designed for
backpropagation training of neural networks~\cite{riedmiller:1993}.
However, it constitutes a general optimization technique. The method
maintains coordinate-wise step sizes $\gamma_i > 0$. In each iteration
the algorithm roughly follows the gradient of the objective function by
evaluating only the signs of the partial derivatives:
\begin{align*}
	w_i^{(t)} = w_i^{(t-1)} - \sign \left( \fp{f(w^{(t-1)})}{w_i} \right) \cdot \gamma_i
\end{align*}
The coordinate-wise step sizes are adjusted by multiplication with a
constant $\eta^+ > 1$ or $\eta^- < 1$ if the signs of derivatives in
consecutive iterations agree or disagree, respectively. This simple
scheme turns out to be highly efficient for many problems. It has been
further refined and evaluated in~\cite{igel:2003}.

It is understood that this algorithm would show poor performance without
online adaptation of $\gamma_i$. First of all it would be restricted to
a fixed grid of values. Even worse, it would be unable to adjust its
initial step size settings to the characteristics of a problem instance.
Thus it would most probably be deemed to making either too large or too
small steps.

\subsection{Evolution Strategies}

Evolution Strategies (ES) are a class of randomized zeroth order
(direct) optimization methods. In the last 15 years these evolutionary
algorithms have evolved into highly efficient optimizers. In each
iteration (called generation in the respective literature) the algorithm
samples one or more search points (called offspring individuals) from a
Gaussian search distribution $\Normal(\mu, \Sigma)$. The mean $\mu$ is
centered either on the best sample so far or on a weighted mean of recent
best samples. In simple ES the covariance matrix is restricted to the
form $\Sigma = \sigma^2 \id$ (with $\id$ denoting the unit matrix in
$\R^n$).
The ``step size'' parameter $\sigma > 0$ turns out to be crucial. Any
fixed choice results in extremely poor search performance. This is an
example of a parameter that cannot be fixed beforehand. Instead is needs
to decay and remain roughly proportional to the distance to the optimum.
Various schemes exist for its online adaptation, e.g., the classic
$1/5$-rule~\cite{rechenberg:1978}. Modern ES treat the whole covariance
matrix $\Sigma \in \R^{n \times n}$ as a free parameter. It is adapted
essentially by low-pass filtered weighted maximum likelihood estimation
of the distribution that generates the most successful recent search
points \cite{hansen:2001,suttorp:2009,glasmachers:2010,IGO-TR}. Online
parameter adaptation is an essential integral building block of these
algorithms. On the other hand, and despite their practical success, some
of these mechanisms lack satisfactory theoretic backup.

It is understood that our discussion of online parameter adaptation
techniques remains incomplete. For example, we did not cover trust
region Newton methods such as the Levenberg-Marquardt algorithm. The
various examples in this section should in any case suffice to
demonstrate that online parameter adaptation is a powerful and sometimes
critical technique for optimization performance. This naturally raises
the question why it has to date not been applied to the coordinate
selection distribution~$\pi$ of the CD algorithm.

\section{Online Adaptation of Coordinate Frequencies}
\label{section:ACF}

In this section we develop an online adaptation method for the CD
coordinate selection distribution~$\pi$. As a first step towards a
practical algorithm we ask the following questions:
\begin{enumerate}
\item
	What is the goal of adaptation?
\item
	Which quantity should trigger adaptation, i.e., which compact
	statistics of the optimization history indicates that adaptation
	is beneficial, and into which direction to adapt?
\end{enumerate}

The answer to the first question seems clear: we'd like to minimize the
runtime of the CD algorithm, or nearly equivalently, to maximize the
pace of convergence to the optimum. Since the optimum is of course
unknown this condition is hard to verify. However, experimentation with
a controlled family of CD problems as done in \cite{glasmachers:2013acml}
(and which was redone in a cleaner fashion, see section~\ref{section:MC})
reveals that this property seems to coincide with an easy-to-measure
statistics: maximization of the rate of convergence on unconstrained
quadratic problems is highly correlated with the fact that on average
the relative progress
\begin{align}
	\frac{f(w^{(t)}) - f(w^{(t-1)})}{f(w^{(t)}) - f^*}
	\label{eq:relative-progress}
\end{align}
(where $f^*$ denotes the optimal objective value) becomes independent of
the coordinate doing the step.%
\footnote{This statement and several related ones in this section are on
an intuitive level; they will be made rigorous in the next section.}
This observation provides us with a powerful tool, namely with the
working assumption that
\begin{itemize}
\item (a) maximizing the convergence rate and
\item (b) making average relative progress equal for all coordinates%
\footnote{Note that we need to exclude trivial solutions such as putting
  all probability mass on one coordinate, which makes progress vanish in
  all coordinates. For technical details refer to the next section.}
\end{itemize}
are equivalent.
We decide for (b) as our answer to the first question in the following.

This results in a straightforward quantity to monitor, namely average
relative progress per coordinate, or more precisely, their differences.
A similar and closely related observation is that increasing $\pi_i$ for
some coordinate~$i$ decreases relative progress, on average. This makes
intuitive sense since progress in that coordinate is split over more
frequent and hence smaller steps. Assuming a roughly monotonically
decreasing relation we should increase $\pi_i$ as soon as relative
progress with coordinate~$i$ is above average and the other way round.
This answers the second question above, namely how to do the adaptation.

Next we turn these concepts into an actual algorithm. Monitoring
relative progress is impossible without knowledge of the optimum.
However, CD algorithm often make relatively little progress per step so
that the denominator in equation~\eqref{eq:relative-progress} can be
assumed to remain nearly constant over a considerable number of
iterations. Thus we may well replace relative progress with absolute
progress $\Delta f = f(w^{(t)}) - f(w^{(t-1)})$, which is just the
numerator of equation~\eqref{eq:relative-progress}.

Formally speaking we have added an assumption that goes beyond the
standard CD algorithm at this point: we assume that the progress
$\Delta f$ can be computed efficiently. It turns out that in many cases
including all examples given in section~\ref{section:cd-in-ml} the
computation of $\Delta f$ is a cheap (constant time) by-product of the
CD step.

It holds $\pi_i \leq 1/n$ for at least one coordinate~$i$ (and usually
for the majority of them), so that for large~$n$ only very few progress
samples can be acquired per coordinate, and we should avoid relying on
too old samples. Therefore we do not perform any averaging of
coordinate-wise progress. Instead we maintain an exponentially fading
record of overall average progress, denoted by~$\overline{r}$, and
compare each single progress sample against this baseline in order to
judge whether progress in coordinate~$i$ is better or worse than
average.

The difference
$\Delta f - \overline{r} = [f(w^{(t)}) - f(w^{(t-1)})] - \overline{r}$
triggers a change of $\pi_i$. The exact quantitative form of this update
is rather arbitrary, and many possible forms should work just fine as
long as the change is into the right direction and the order of
magnitude of the change is reasonable. For efficiency reasons we do not
represent $\pi_i$ in the algorithm directly, instead we adapt
unnormalized preferences $p_i$, track their sum
$p_\text{sum} = \sum_{i=1}^n p_i$, and define
$\pi_i = p_i / p_\text{sum}$. Then we manipulate $p_i$ according to the
update rule
\begin{align*}
	p_i \leftarrow \left[ \exp \left( c \cdot \left( \frac{\Delta f}{\overline{r}} - 1 \right) \right) \cdot p_i \right]_{p_{\min}}^{p_{\max}}
\end{align*}
where ${p_{\min}}$ and ${p_{\max}}$ are lower and upper bounds and
$[t]_a^b = \min\{\max\{t, a\}, b\}$ denotes clipping of $t$ to the
interval $[a, b]$.
Given a coordinate~$i$ and its single step progress $\Delta f$ this
update step is made formal in algorithm~\ref{algo:ACF}. We call it the
\textit{Adaptive Coordinate Frequencies} (ACF) method. Its parameters
are the lower and upper bounds ${p_{\min}}$ and ${p_{\max}}$ and the
learning rates $c$ for preference adaptation and the exponential fading
record $\overline{r}$ of average progress. Default values for these
parameters are given in table~\ref{table:defaults}. The algorithm state
consists of $\pi$ and $\overline{r}$. The former can be initialized to
the uniform distribution unless a more informed setting is available,
the latter can be initialized to the average progress over a brief
warm-up phase (without adaptation), i.e., a single sweep over the
coordinates.

\begin{algorithm}[h]
\begin{algorithmic}
	\STATE $p_\text{new} \leftarrow \left[ \exp \big( c \cdot (\Delta f / \overline{r} - 1) \big) \cdot p_i \right]_{p_{\min}}^{p_{\max}}$
	\STATE $p_\text{sum} \leftarrow p_\text{sum} + p_\text{new} - p_i$
	\STATE $p_i \leftarrow p_\text{new}$
	\STATE $\overline{r} \leftarrow (1 - \eta) \cdot \overline{r} + \eta \cdot \Delta f$
	\caption{Adaptive Coordinate Frequencies (ACF) Update}
	\label{algo:ACF}
\end{algorithmic}
\end{algorithm}

\begin{table}[h]
\begin{center}
	\setlength{\tabcolsep}{1em}
	\begin{tabular}{l|c}
		\textbf{parameter} & \textbf{value} \\
		\hline
		$c$ & $1/5$ \\
		$p_{\min}$ & $1/20$ \\
		$p_{\max}$ & $20$ \\
		$\eta$ & $1/n$
	\end{tabular}\\[1em]
	\vspace{1em}
	\caption{ \label{table:defaults}
		Default parameter values for the ACF algorithm.
		These values were set rather ad-hoc; in particular
		they did not undergo extensive tuning. The algorithm
		was found to be rather insensitive to these settings.
	}
\end{center}
\end{table}

Until now we did not specify how samples are drawn from~$\pi$. I.i.d.\
coordinate selection is the simplest possibility, however, it requires
$\Theta(\log(n))$ time per sample \cite{nesterov:2012}. This is in
contrast to uniform selection, the time complexity of which is
independent of~$n$. Ideally we would like to achieve the same for an
arbitrary distribution~$\pi$. This can be done by relaxing the i.i.d.\
assumption and instead selecting coordinates in blocks of size
$\Theta(n)$. Here we present a deterministic variant for drawing
$\Theta(n)$ samples from $\pi$ in $\Theta(n)$ operations, hence with
amortized constant time complexity per CD iteration.
Despite drawing a finite set of indices algorithm~\ref{algo:schedule}
respects the exact distribution~$\pi$ over time with the help of
accumulator variables $a = (a_1, \dots, a_n)$.

\begin{algorithm}[h]
\begin{algorithmic}
	\STATE $J \leftarrow \{\}$
	\FOR{$i \in I$}
		\STATE $a_i \leftarrow a_i + n \cdot p_i / p_\text{sum}$
		\STATE $\lfloor a_i \rfloor$ times: append index $i$ to list $J$
		\STATE $a_i \leftarrow a_i - \lfloor a_i \rfloor$
	\ENDFOR
	\STATE shuffle list $J$
	\caption{Creation of a sequence $J$ of coordinates according to $\pi$}
	\label{algo:schedule}
\end{algorithmic}
\end{algorithm}

The algorithm outputs a sequence of on average $n$ and at most
$2 \cdot n$ coordinates at a time at a cost of $\Theta(n)$ operations
while guaranteeing that each coordinate has a waiting time of at most
\begin{align*}
	\lceil 1 / (n \cdot p_i) \rceil \leq \lceil 1 / (n \cdot p_{\min}) \rceil = \tau < \infty
\end{align*}
sweeps for its next inclusion. This property guarantees convergence of
the resulting CD algorithm with the same arguments as in the proof of
theorem~1 by \cite{hsieh:2008}, which is based on theorem~2.1 in
\cite{luo:1992}. Alternatively, in the terminology of
Tseng~\cite{tseng:2001} algorithm~\ref{algo:schedule} realizes an
\textit{essentially cyclic rule} for coordinate selection.
Thus, our ACF-CD algorithm enjoys the same convergence guarantees as
other CD schemes with fixed, e.g., cyclic coordinate selection.

\section{Randomized CD as a Markov Chain}
\label{section:MC}

In this section we analyze the qualitative behavior of the ACF algorithm.
We formalize most of the intuition presented in the previous section in
terms of (properties of) Markov chains.

In a first step we capture the behavior of the CD algorithm for a fixed
distribution~$\pi$. Then we formalize a central conjecture in
mathematical terms and show---based on this assumption---that in
expectation the ACF method drives this distribution into the vicinity of
the optimal distribution.

We perform this analysis for an unconstrained quadratic problem
\begin{align*}
	\min_{w \in \R^n} \quad f(w) = \frac{1}{2} w^T Q w
\end{align*}
with strictly positive definite, symmetric Hessian matrix
$Q \in \R^{n \times n}$.

A particularly simple (e.g., diagonal) structure of $Q$ allows to locate
the optimum exactly after finitely many iterations (e.g., after exactly
$n$ iteration with a cyclic coordinate selection rule). In this case
an exact runtime analysis is trivial, but this case does not play any
role in practice.
In the remainder of this section we consider the more relevant case of
an infinite chain: we assume $P(w^{(t)} = 0) = 0$ for all $t \in \N$.
This will allow us to ``divide by $f(w)$''.

The unconstrained quadratic problem is relevant for the understanding of
the convergence speed of CD on a large class of optimization problems.
For this sake every twice continuously differentiable objective function
$f$ can be well approximated in the vicinity of the optimum by its
second order Taylor polynomial, and under mild technical assumptions it
can be assume that after some iterations $t_0$ all constraints either
remain active or inactive so that the problem can be treated essentially
as an unconstrained problem on the free variables. In the context of SVM
optimization such an argument is found e.g.\ in \cite{fan:2005} (based
on an earlier result by Lin~\cite{lin:2001}).

\subsection{The CD Markov Chain for fixed $\pi$}

We start with the analysis for fixed~$\pi$, i.e., without ACF.
In each iteration $t \in \N$ the algorithm picks an index
$i^{(t)} \in I$ according to a predefined distribution $\pi$ on $I$ and
then solves the one-dimensional sub-problem in $w_{i^{(t)}}$ optimally%
\footnote{Optimality of course refers to single-step behavior, i.e.,
  the algorithm solves the one-dimensional sub-problem in a greedy
  manner.}
with a one-dimensional Newton step
\begin{align*}
	w^{(t)}_{i^{(t)}} &= w^{(t-1)}_{i^{(t)}} - \frac{Q_{i^{(t)}}^T w^{(t-1)}}{Q_{i^{(t)},i^{(t)}}}
	\enspace,
\end{align*}
where $Q_i$ denotes the $i$-th column of $Q$. This iteration scheme is
expressed equivalently in vector notation as
$w^{(t)} = T_{i^{(t)}} w^{(t-1)}$ with
\begin{align*}
	T_i = \bmat 1 & & \dots & & 0 \\ \vdots & \ddots & & & \vdots \\ -\frac{Q_{i1}}{Q_{ii}} & \dots & 1 - \frac{Q_{ii}}{Q_{ii}} = 0 & \dots & -\frac{Q_{in}}{Q_{ii}} \\ \vdots & & & \ddots & \vdots \\ 0 & & \dots & & 1 \emat
	\enspace.
\end{align*}
The matrix $T_i$ fulfills $T_i^2 = T_i$; it defines a projection onto
the hyperplane $H_i = \big\{w \in \R^n \,\big|\, Q_i^T w = 0 \big\}$.
The transition operator $T : \R^n \to \R^n$, $T(w) = T_i w$ with
$i \sim \pi$ performs one iteration of the randomized CD algorithm.%
\footnote{The square matrices $T_i$ are linear operators on states, not
  on probability distributions. They are not to be confused with
  transition matrices of Markov chains on finite state spaces, even if
  their role is similar.}

For convenience, the distribution $\pi$ is represented as an element of
the probability simplex
\begin{align*}
	\Delta = \left\{ p \in \R^n \,\left|\, p_i \geq 0 \text{ and } \sum_{i=1}^n p_i = 1 \right. \right\}
	\enspace.
\end{align*}
With $\mathring{\Delta}$ we denote the interior of the simplex. Thus
$\pi \in \mathring{\Delta}$ is equivalent to $\pi_i > 0$ for all
$i \in I$.

The problem instance~$Q$ and the distribution~$\pi$ define a time
homogeneous Markov chain $w^{(t)} \in \R^n$ with random transition
operator~$T$. We start out be collecting elementary properties of this
chain.

Assume the optimization works as expected then the chain $w^{(t)}$
converges to the optimum. Thus the only stationary limit distribution
should be a Dirac peak over the optimum. This distribution does not
provide any insights into the actual optimization process. One way of
describing the regularity of the process is by considering the
distribution of \emph{directions} from which the optimum is approached.
This property can be captured by a scale invariant state description.
In the following we construct a scale-invariant Markov chain with a
non-trivial limit distribution.

\begin{lemma} \label{lemma:scale-invariance}
The Markov chain is scale invariant, i.e., the transition operator
commutes with scaling by any factor $\alpha \not= 0$.
\end{lemma}
\begin{proof}
We have to show that $T(\alpha \cdot w) = \alpha \cdot T(w)$. This is a
trivial consequence of the linearity of $T_i$, since application of $T$
amounts to the application of a random $T_i$, all of which are linear
operators. \qed
\end{proof}
Scaling the initial solution $w^{(0)}$ by a scalar factor
$\alpha \not= 0$ results in the chain $\alpha \cdot w^{(t)}$. Hence the
projection of the chain onto the projective space $\Projective(\R^n)$ is
well-defined. The projective space is the ``space of lines'', i.e., the
space of equivalence classes of the relation
$w \sim w' \Leftrightarrow w = \alpha \cdot w'$ for some
$\alpha \not= 0$ on $\R^n \setminus \{0\}$. Equivalently, the projective
space is obtained by identifying antipodal points on the sphere; it is
thus compact.
We denote the corresponding chain of equivalence classes (lines) by
$z^{(t)} = \kappa(w^{(t)})$. Here
$\kappa : \R^n \setminus \{0\} \to \Projective(\R^n)$,
$\kappa(w) = (\R \setminus \{0\}) \cdot w$
denotes the canonical projection.

Any CD step with coordinate $i \in I$ ends on the hyperplane $H_i$:
$T_i w \in H_i$ for all $i \in I$ and $w \in \R^n$, and hence
$w^{(t)} \in H_{i^{(t)}}$. Let $\mu^{(t)}$ denote the distribution of
$w^{(t)}$. The support of the distribution $\mu^{(t)}$, $t \in \N$, is
restricted to the union $H = \bigcup_{i=1}^n H_i$ of the hyperplanes $H_i$.
The distribution can hence be written as a superposition
$\mu^{(t)} = \sum_{i=1}^n \pi_i \mu_i^{(t)}$, where each $\mu_i^{(t)}$
is a distribution on~$H_i$.

\begin{lemma} \label{lemma:linear-convergence}
Consider $\pi \in \mathring{\Delta}$. For each
$w \in \R^n \setminus \{0\}$ we define the expected one-step
progress rate $r(\pi, w) = \Expectation[f(T(w))] / f(w)$.
Then there exists a constant $U_{\pi} < 1$ such that it holds
$r(\pi, w) \leq U_{\pi}$ for all $w \in \R^n$. In other words the
expected distance to the optimal value $\Expectation[f(w^{(t)})]$
converges to zero at least at a linear rate of $U_{\pi}$.
\end{lemma}
\begin{proof}
For each $w$ progress can be made in at least one coordinate $i \in I$
and because of $\pi_i > 0$ it holds $r(\pi, w) < 1$. The function
$r : \mathring{\Delta} \times (\R^n \setminus \{0\}) \to [0, 1)$ depends
continuously on $w$ and on $\pi$. Furthermore we have
$r(\pi, w) = r(\pi, \alpha \cdot w)$ for all $\alpha \not= 0$ by scale
invariance, which means that $r(\pi, \cdot)$ can be lifted to
$\Projective(\R^n)$, the compactness of which implies that the supremum
\begin{align*}
	U_{\pi} = \!\!\! \sup_{w \in \R^n \setminus \{0\}} \Big\{ r(\pi, w) \Big\}
\end{align*}
is attained. It follows $U_{\pi} < 1$. \qed
\end{proof}

We are interested in the dependency of the rate of convergence on the
distribution $\pi$ since we aim to eventually improve or even maximize
the progress rate of the CD algorithm.

For all $t \in \N$ we define the mixture distributions
$\nu^{(t)} = \sum_{i=1}^n \pi_i \nu_i^{(t)}$ on $\Projective(\R^n)$,
defined by $\nu_i^{(t)}(E) = \mu_i^{(t)} \big( \kappa^{-1}(E) \big)$
for all measurable $E \subset \Projective(\R^n)$.
The support of $\nu^{(t)}_i$ is restricted to
$\kappa(H_i) \subset \Projective(\R^n)$.
By definition it holds $z^{(t)} \sim \nu^{(t)}$.
Our further analysis is based on the fact that the scale invariant
component $z^{(t)}$ inherits the Markov property.

\begin{lemma} \label{lemma:normalized-markov}
The scale-invariant variables $z^{(t)}$ form a time-homogeneous Markov
chain on the compact space $\Projective(\R^n)$.
\end{lemma}
\begin{proof}
We show that the transition operator~$T$ lifted to
$\Projective(\R^n)$ is well-defined, i.e., that it holds
$\kappa(w) = \kappa(w') \Rightarrow \kappa(T(w)) = \kappa(T(w'))$. This
is a trivial consequence of scale invariance: $\kappa(w) = \kappa(w')$
implies the existence of $\alpha \not= 0$ such that
$w' = \alpha \cdot w$, and hence
\begin{align*}
	\kappa(T(w')) &= \kappa(T(\alpha \cdot w)) \\
		&= \kappa(\alpha \cdot T(w)) \\
		&= \kappa(T(w))
	\enspace.
\end{align*}
Let $T'$ denote the now well-defined lift of the transition operator,
and $T'_i$ the corresponding step with coordinate index $i$. Then
$z^{(t)} = T'(z^{(t-1)})$ depends on the chain's history only through
its predecessor state. Time-homogeneity of $z^{(t)}$ is a direct
consequence of time-homogeneity of~$w^{(t)}$. \qed
\end{proof}

Under weak technical assumptions $\nu^{(t)}$ converges to a stationary
distribution $\nu^{\infty}$, e.g., by excluding exact cycles of
$z^{(t)}$.%
\footnote{This is another minor technical prerequisite on the problem
  instance~$Q$; it essentially excludes a zero set of instances.}
It is important to note that the stationary distribution is \emph{not
independent} of the initial state~$z^{(0)}$.

\begin{lemma}
The distribution $\nu^{\infty}$ is of the form
$\nu^{\infty} = \sum_{i=1}^n \pi_i \nu_i^{\infty}$, with the support of
$\nu_i^{\infty}$ restricted to $\kappa(H_i)$.
\end{lemma}
\begin{proof}
The form of $\nu^{\infty}$ is an elementary consequence of the
homogeneous forms of $\nu^{(t)} = \sum_{i=1}^n \pi_i \nu_i^{(t)}$ with
the same coefficients $\pi_i$ for all $t \in \N$. \qed
\end{proof}

The quantity of ultimate interest is the progress of the chain $w^{(t)}$
towards the optimum while the projection $z^{(t)}$ converges to its
stationary distribution. This progress rate is captured as follows.
We define the coefficients
\begin{align*}
	\rho_{ij} = \Expectation_{z \sim \nu_i^{\infty}} \Big[ \log(f(w)) - \log(f(T_j w)) \Big]
\end{align*}
measuring average progress of transitions from $H_i$ to $H_j$. Note
that $\log(f(w)) - \log(f(T_j w))$ is invariant under
scaling of $w$ and thus well-defined given $z = \kappa(w)$. The
aggregations
\begin{align}
	\rho_i = \Expectation_{j \sim \pi} \big[ \rho_{ij} \big] &= \sum_{j=1}^n \pi_j \rho_{ij} \notag \\
	\rho = \Expectation_{i, j \sim \pi} \big[ \rho_{ij} \big] &= \sum_{i,j=1}^n \pi_i \pi_j \rho_{ij} \label{eq:rho}
\end{align}
measure average progress of steps with coordinate~$i \in I$ and overall
average progress, respectively.

\begin{lemma}
It holds
\begin{align*}
	\rho &= \lim_{t \to \infty} \,\, \frac{1}{t} \cdot \Expectation \Big[ \log(f(w^{(0)})) - \log(f(w^{(t)})) \Big] \\
	     &= \lim_{t \to \infty} \,\, \frac{1}{t} \cdot \Big[ \log(f(w^{(0)})) - \log(f(w^{(t)})) \Big]
	\enspace.
\end{align*}
where the last equality holds almost everywhere.
\end{lemma}
\begin{proof}
The iterates take the form $w^{(t)} = T^t w^{(0)}$ with $T$ being random
matrices from the set $\{T_1, \dots, T_n\}$, distributed according
to~$\pi$. The multiplicative ergodic theorem by Oseledec (theorem~1.6
and corollary~1.7 in~\cite{ruelle:1979}) guarantees that the limits of
the sequences $\frac{1}{t} \log(\|T^t\|)$ (with $\|\cdot\|$ denoting a
sub-multiplicative matrix norm) and $\frac{1}{t} \log(\|T^t w^{(0)}\|)$
exist a.e.
An application with norm $\|w\|_Q = \sqrt{w^T Q w} = \sqrt{f(w)}$ (and
the induced matrix norm) gives the second equality. The first equality
is an immediate consequence of the definition of $\rho$ and the fact
that the chain $z^{(t)}$ converges to its equilibrium distribution. \qed
\end{proof}

The relation $\exp(-\rho) \leq U_{\pi} < 1$ is obvious, where $U_{\pi}$
is the constant defined in lemma~\ref{lemma:linear-convergence}. The
asymptotic convergence rate is given by $\exp(-\rho)$, while $U_{\pi}$
is a (non-asymptotic) upper bound.

\subsection{Optimal Coordinate Distribution and ACF}

The goal of coordinate frequency adaptation is to maximize the pace of
convergence, or equivalently to maximize $\rho$. In this context we
understand $\rho = \rho(\pi)$ as a function of $\pi$ with an implicit
dependency on the (fixed) problem instance~$Q$. We aim for an adaptation
rule that drives $\pi$ towards a maximizer of~$\rho$.

The CD algorithm converges to the optimum (with full probability) for
all interior points $\pi \in \mathring{\Delta}$ (since all coordinates
are selected arbitrarily often), and it converges to a sub-optimal point
for boundary points $\pi \in \partial \Delta$ (since at least one
coordinate remains fixed). We conclude that boundary points cannot be
maximizers of $\rho$. Hence continuity of $\rho$ implies the existence
of a maximizer
$\pi^* \in \arg\max_{\pi} \big\{ \rho(\pi) \big\} \subset \mathring{\Delta}$
in the interior of the simplex.

The identification of the problem-dependent distribution $\pi^*$ is thus
the goal of ACF online adaptation. This could be attempted by
maximization of equation~\eqref{eq:rho}, which requires a descent
understanding of the Markov chain $z^{(t)}$ and its stationary
distribution~$\nu^{(\infty)}$.

It turns out that most standard tools for the analysis of continuous
state space Markov chains are not applicable in this case. For example,
the chain is not $\phi$-irreducible%
\footnote{This is easy to see from the fact that only countably many
  points are reachable from each initial point.}
and thus we cannot hope that its stationary distribution $\nu^{\infty}$
is independent of the starting state~$z^{(0)}$---however, we conjecture
that the resulting progress rate $\rho$ is. The exact functional
dependency of $\rho$ on $\pi$ (and on $Q$) turns out to be complicated,
and its detailed analysis is beyond the scope of this paper. This
situation excludes direct maximization of equation~\eqref{eq:rho}.
Instead we propose an indirect way of identifying $\pi^*$ by means of
the following conjecture, which is a formalization of the empirical
observation presented in section~\ref{section:ACF}:

\newtheorem{conj}{Conjecture}

\begin{conj} \label{conjecture:equilibrium}
The maximizer $\pi^*$ of the progress rate $\rho(\pi)$ is the only
distribution in $\mathring{\Delta}$ that fulfills the equilibrium
condition $\rho_i(\pi^*) = \rho(\pi^*)$ for all $i \in I$.
\end{conj}
This conjecture gives a relatively easy to test condition for the
identification of $\pi^*$ without the need for a complete understanding
of the underlying Markov chain.

We proceed by testing the conjecture numerically. For this purpose we
simulate the Markov chain $z^{(t)}$ over extended periods of time. This
allows for accurate measurement of $\rho$. The coordinate-wise components
$\rho_i$ can be measured accordingly. However, for high dimensions $n$
the accuracy of these measurements becomes poor because for at least one
$i \in I$ the number of samples available for the estimation of $\rho_i$
is at least $n$ times lower than for $\rho$. Thus a numerical test of
the conjecture is feasible only for small~$n$.

We have performed experiments with random matrix instances~$Q$ in
dimensions $n \in \{4, 5, 6, 7\}$.
Random problem instances $Q$ were created as follows: A set of $n$
points $x_i \in \R^2$ was drawn i.i.d.\ from a standard normal
distribution. The matrix $Q$ was then defined as the kernel Gram matrix
of these points w.r.t.\ the Gaussian RBF kernel function
\begin{align*}
	Q_{ij} = k(x_i, x_j) = \exp \left( \frac{\|x_i - x_j\|^2}{2 \sigma^2} \right)
\end{align*}
for $\sigma = 3$. This model problem is related to learning and
optimization problems arising in kernel-based machine learning. Other
choices of $Q$, e.g., as a product $Q = A^T A$ with standard normally
distributed entries $A_{ij}$ gave similar results.

Starting from a uniform distribution we have adjusted $\pi$ so as to
balance the coordinate-wise progress rates $\rho_i$. This was achieved
by adaptively increasing $\pi_i$ if $\rho_i > \rho$ and decreasing
$\pi_i$ if $\rho_i < \rho$ with an Rprop-style algorithm
\cite{riedmiller:1993}. We denote the resulting distribution by
$\overline{\pi}$. Then we have systematically varied this distribution
along $n$ curves $\gamma_{\overline{\pi}, i}(t)$ through the probability
simplex, defined as
\begin{align*}
	\tilde \gamma_{\pi, i}(t) &= \pi + (2^t - 1) \pi_i e_i \\
	\gamma_{\pi, i}(t) &= \frac{1}{\| \tilde \gamma_{\pi, i}(t) \|} \cdot \tilde \gamma_{\pi, i}(t)
	\enspace,
\end{align*}
where $e_i$ is the $i$-th unit vector. Values
$t \in \big\{ -1, -\frac{1}{2}, -\frac{1}{4}, -\frac{1}{10}, 0, \frac{1}{10}, \frac{1}{4}, \frac{1}{2}, 1 \big\}$
were chosen for evaluation. The progress rate $\rho(\pi)$ was estimated
for each of these distributions numerically by simulating the Markov
chain until an estimate of the standard deviation of $\rho$ fell below
a threshold of $10^{-4} \cdot \rho$. The resulting one-dimensional
performance curves are displayed in figure~\ref{fig:rho-pi}. It turns
out that the maximum is attained at position $t = 0$, corresponding to
$\gamma_{\overline{\pi}, i}(t) = \overline{\pi}$. All curves are
uni-modal with a single maximum, clearly hinting at
$\overline{\pi} \approx \pi^*$.

\begin{figure}[htp]
	\includegraphics[width=0.49\textwidth]{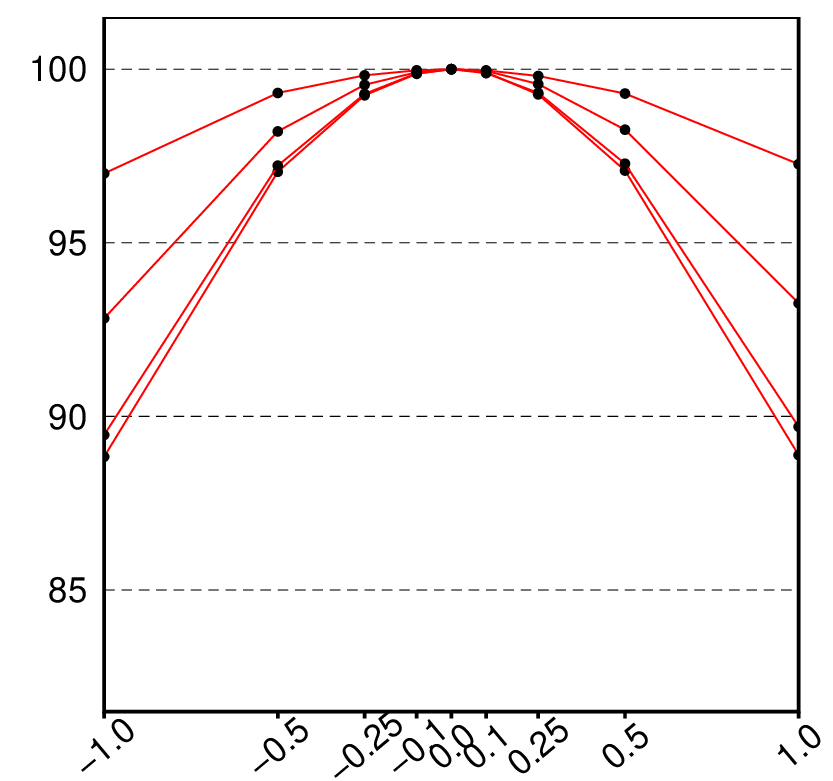}~\includegraphics[width=0.49\textwidth]{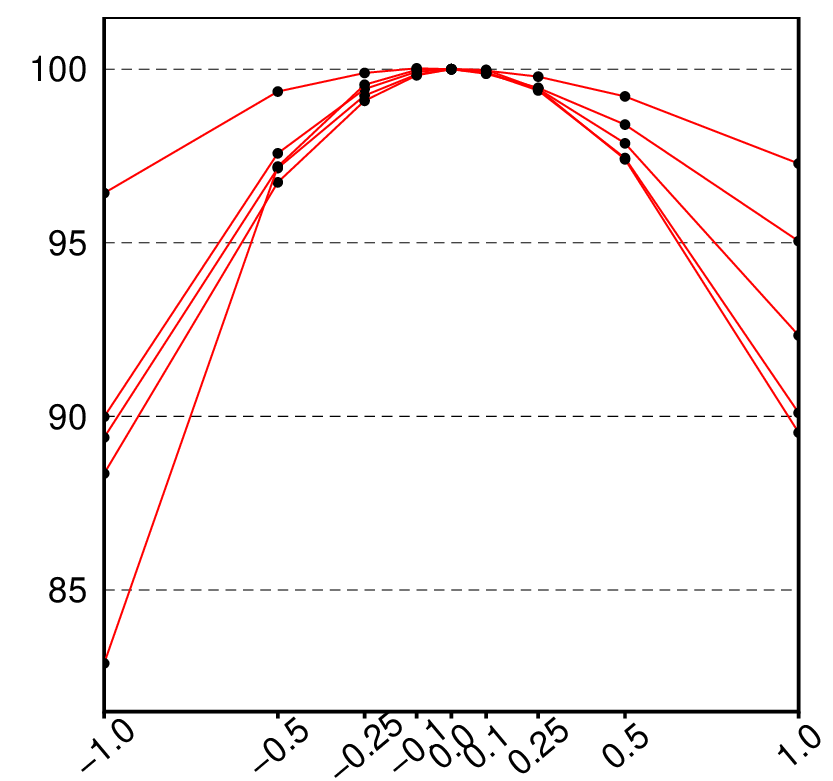}\\[0.5em]
	\includegraphics[width=0.49\textwidth]{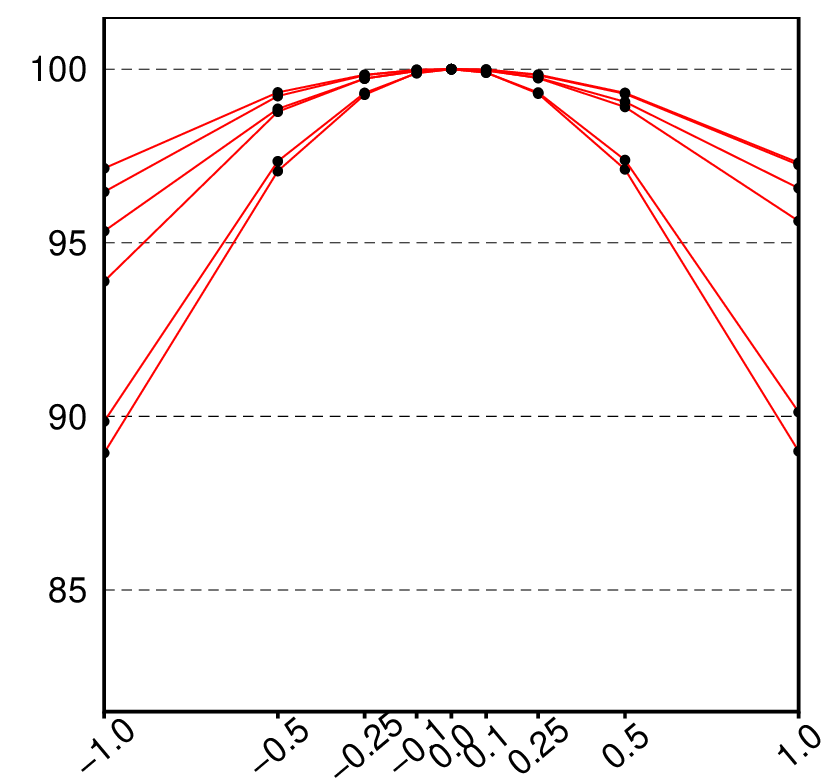}~\includegraphics[width=0.49\textwidth]{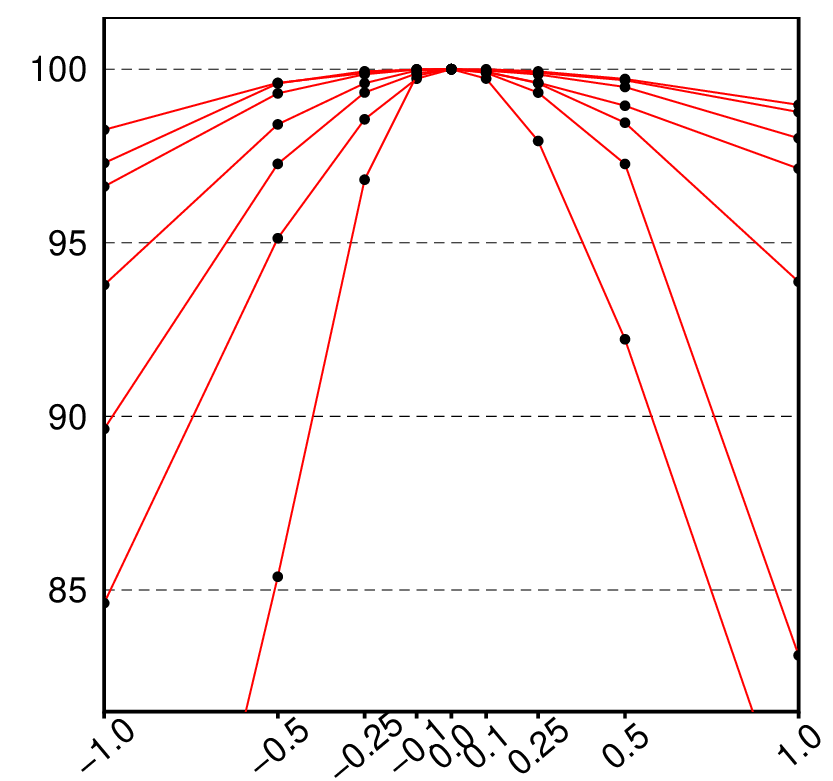}
\begin{center}
	\caption{ \label{fig:rho-pi}
		Curves $t \mapsto \rho(\gamma_{\overline{\pi}, i}(t)) / \rho(\overline{\pi})$,
		$i \in I = \{1, \dots, n\}$, for random problem instances in
		dimensions $n=4$ (top left), $n=5$ (top right),
		$n=6$ (bottom left), and $n=7$ (bottom right).
		The numerical optimum is located at $t = 0$, corresponding to
		the distribution~$\overline{\pi} \approx \pi^*$.
	}
\end{center}
\end{figure}

These experiments indicate that conjecture~\ref{conjecture:equilibrium}
may indeed hold true. In any case we argue that the empirical evidence
is sufficient to justify the design of a heuristic online adaptation
strategy for $\pi$ that is based on equalizing all $\rho_i$. Moreover,
the uni-modality of the performance curves indicates that the
identification of the global optimum may be possible with iterative
methods starting from any initial configuration.

The ACF algorithm performs a similar type of adaptation as described
above for obtaining an estimate of $\overline{\pi}$, with the decisive
difference that this adaptation is performed online and without
knowledge of the optimum (which would be necessary for the computation
of $\rho$ and $\rho_i$). It aims at maximizing $\rho$ by driving the
coordinate-wise progress deviations $\rho_i - \rho$ towards zero.

The intuition behind ACF's adaptation mechanism is as follows: first of
all, increasing the $i$-th coordinate's probability $\pi_i$ results in a
decrease of its progress rate. This is because the progress in
direction~$i$ is spread over more CD steps. In the extreme case of
performing two consecutive steps with the same coordinate the second
step does not make any further progress (provided that one-dimensional
sub-problems are solved optimally). The second insight is that there is
no need to compute $\log(f(w))$ (or more generally $\log(f(w) - f^*)$
where $f^*$ denotes the unknown optimal objective value) in order to
compare coordinate-wise progress. Instead it is sufficient to compare
the step-wise \emph{gains} $f(w^{(t-1)}) - f(w^{(t)})$ for different
coordinates, or equivalently, to compare them to the moving average
$\overline{r}$ in algorithm~\ref{algo:ACF}.
Adjusting $\pi_i$ so that all coordinate-wise gains become equal should
be about the same as equalizing coordinate-wise progress rates. This
intuition is made explicit in the following theorem. It ensures that
under a number of conditions the ACF algorithm indeed adapts the
expected coordinate distribution $\pi$ so as to maximize the progress
rate.

\begin{theorem} \label{theorem:acf}
Assume the preconditions
\begin{enumerate}
\item
	the Markov chain $z^{(t)}$ is in its stationary distribution,
\item
	the progress rate $\rho$ is infinitesimal, or in other words,
	the first order Taylor approximation
	\begin{align*}
		\log(f(w)) - \log(f(T_i(w))) \approx \frac{f(w) - f(T_i(w))}{f(w)}
	\end{align*}
	in $w$ becomes exact,
\item
	the estimate $\overline{r}$ of average progress in
	algorithm~\ref{algo:ACF} is exact, i.e.,
	$\overline{r} = \Expectation[f(w) - f(T(w))]$.
\item
	$\rho_i \big( \gamma_{\pi,i}(t) \big)$ is strictly monotonically
	decreasing for all $\pi \in \mathring{\Delta}$ and $i \in I$.
\end{enumerate}
Let $\pi^{(t)}$ denote the sequence of distributions generated by the
ACF algorithm with learning rate $0 < \eta \ll 1$ and bounds
$p_{\min} = 0$ and $p_{\max} = \infty$. Then $\Expectation[\pi^{(t)}]$
fulfills the equilibrium condition $\rho_i = \rho$ for all $i \in I$.
\end{theorem}
\begin{proof}
From the prerequisites we obtain
\begin{align*}
	\rho_i &= \Expectation \big[ \log(f(w)) - \log(f(T_i(w))) \big] \\
		&= \Expectation \left[ \frac{f(w) - f(T_i(w))}{f(w)} \right] \\
		&= c \cdot \Expectation \big[ f(w) - f(T_i(w)) \big]
\end{align*}
with constant of proportionality $c = 1 / f(w)$, which is quasi
constant due to prerequisite~2, and
\begin{align*}
	\rho = \Expectation \big[ \rho_i \big] = c \cdot \Expectation \big[ f(w) - f(T(w)) \big] = c \cdot \overline{r}
	\enspace.
\end{align*}
We conclude that in expectation the deviation of additive progress
$f(w) - f(T_i(w))$ from its mean $\overline{r}$ is proportional to the
deviation of coordinate-wise progress $\rho_i(\pi)$ from its
mean~$\rho(\pi)$:
\begin{align*}
	\Expectation \Big[ \big( f(w) - f(T_i(w)) \big) - \overline{r} \Big] \quad \propto \quad \rho_i(\pi) - \rho(\pi)
	\enspace.
\end{align*}
In the ACF algorithm the term inside the expectation on the left hand
side drives the adaptation of $\pi$, which again impacts the right
hand side according to prerequisite~4. Thus, in expectation and for a
small enough learning rate $\eta$ the ACF rule drives the distribution
towards the equilibrium distribution~$\pi^*$. Any stationary point of
this process fulfills $\big( f(w) - f(T_i(w)) \big) = \overline{r}$ for
all $i \in I$, which is equivalent to $\rho_i = \rho$ for all $i \in I$.
\qed
\end{proof}

None of the prerequisites of theorem~\ref{theorem:acf} is a strong
assumption, indeed, all of them should be fulfilled in practice in good
approximation: CD algorithms usually perform extremely many cheap update
steps so the chain has enough time to approach its stationary
distribution, with the same arguments gains per single step are small,
and the monotonicity property~4 can easily be validated in the very same
experiments that test the conjecture.
It is also clear that some prerequisites are only fulfilled
approximately. Thus the ACF algorithm's stationary distribution is only
a proxy of the ideal distribution $\pi^*$. However, the deviation of the
resulting progress rates is usually small since the function $\rho$ is
rather flat in the vicinity of its optimizer~$\pi^*$
(see figure~\ref{fig:rho-pi}).

A reasonably careful interpretation of theorem~\ref{theorem:acf} is that
provided conjecture~\ref{conjecture:equilibrium} holds
the ACF algorithm adjusts the distribution $\pi$ in a close to optimal
way, on average. Note that the statement gives no guarantees on the
variance of $\pi^{(t)}$. For large dimensions~$n$ this variance can be
significant since only few samples are available for the estimation of
$\rho_i$. This is why in rare cases ACF may even have a deteriorating
effect on performance in practice as can be seen from some of the
experimental results in the next section.

\section{Empirical Evaluation}

In this section we investigate the performance of the ACF method. We
have run algorithm~\ref{algo:CD} in a number of variants reflecting the
state-of-the-art in the respective fields against the ACF-CD algorithm
for solving a number of instances of problem~\eqref{eq:trainingproblem},
namely the four problems discussed in section~\ref{section:cd-in-ml}.

We have implemented the ACF algorithm directly into the software
liblinear \cite{liblinear} for binary SVM and logistic regression
training of linear models. An efficient C implementation was created for
the LASSO problem with quadratic loss. We have implemented a CD solver
for linear multi-class training into the Shark machine learning
library~\cite{shark:2008}.

The stopping criteria for all algorithms were set in analogy to the
standard stopping criterion for linear and non-linear SVM training as
found in libsvm~\cite{fan:2005}, liblinear~\cite{liblinear}, and
Shark~\cite{shark:2008}. In the case of dual SVM training the algorithm
is stopped as soon as all Karush-Kuhn-Tucker (KKT) violations drop below
a threshold $\varepsilon$. The default value for non-linear SVM training
has been established as $\varepsilon = 0.001$, and less tight values
such as $\varepsilon = 0.01$ and even $\varepsilon = 0.1$ are commonly
applied for fast training of linear models. For SVMs (without bias term)
this measure simply computes the largest absolute component of the
gradient of the dual objective that is not blocked by an active
constraint. For (essentially) unconstrained problems such as LASSO and
logistic regression the stopping criterion checks whether all components
of the dual gradient have dropped below~$\varepsilon$.

The number of CD iterations is a straightforward performance indicator.
For sparse data this indicator is not always a reliable indicator of
computational effort, since not all coordinates correspond to roughly
equal numbers of non-zeros. In fact, in the LASSO experiments the cost
of the derivative computation that dominates the cost of a CD iteration
varies widely and cannot be assumed to be roughly constant.
Wall clock optimization time is a more relevant measure.
We measure optimization time only for exactly comparable implementations
and otherwise resort to the number of multiplications and additions
required to compute the derivatives, hereafter referred to as the
\textit{number of operations}. This quantity is a very good predictor of
the actual runtime, with the advantage of being independent of
implementation, CPU, memory bandwidth, and all kinds of inaccuracies
associated with runtime measurements.

Comparing training times in a fair way is non-trivial. This is because
the selection of a good value of the regularization parameter ($\lambda$
or $C$) requires several runs with different settings, often performed
in a cross validation manner. The computational cost of finding a good
value can easily exceed that of training the final model, and even a
good range is often hard to guess without prior knowledge.
The focus of the present study is on optimization. Therefore we don't
fix a specific model selection procedure and instead report training
times over reasonable ranges of values.

The various data sets used for evaluation are available from the libsvm
data website
\begin{center}
	\texttt{http://www.csie.ntu.edu.tw/{\textasciitilde}cjlin/libsvmtools/datasets/}
	\enspace.
\end{center}
The complete source code of our experiments is available at
\begin{center}
	\texttt{http://www.ini.rub.de/PEOPLE/glasmtbl/code/acf-cd/}
	\enspace.
\end{center}

\subsection{LASSO Regression}

To demonstrate the versatility of our approach we furthermore compared
ACF-CD to the LASSO solver proposed by \cite{friedman:2007}.
This is a straightforward deterministic CD algorithm, iterating over all
coordinates in order. We have used the data sets listed in
table~\ref{table:data-LASSO}, and for each of these problems we have
varied the parameter $\lambda$ in a range so that the resulting number
of non-zero features varies between very few (less than $10$) and many
(more than $10,000$), covering the complete range of interest.
This gives a rather complete picture of the relative performance of both
algorithms over a wide range of relevant optimization problems. The
results are summarized in table~\ref{tab:lasso-results}.

\begin{table}[t]
\begin{center}
	\setlength{\tabcolsep}{0.5em}
	\begin{tabular}{l|r|r}
		Problem & Instances $(\ell)$ & Features $(n = d)$ \\
		\hline
		news~20       & $19,996$ & $1,355,191$ \\
		rcv1          & $20,242$ & $47,236$ \\
		E2006-tfidf   & $16,087$ & $150,360$ \\
	\end{tabular}
	\vspace{1em}
	\caption{ \label{table:data-LASSO}
		Benchmark problems for the LASSO experiments.
	}
\end{center}
\end{table}

\begin{table}[h]
\begin{center}
	\setlength{\tabcolsep}{0.5em}
	\begin{tabular}{l|c|cc|cc|rr}
		\textbf{problem} & $\lambda$ & \multicolumn{2}{|c}{\textbf{uniform}} & \multicolumn{2}{|c}{\textbf{ACF}} & \multicolumn{2}{|c}{\textbf{speed-up}} \\
		  &  & iterations & operations & iterations & operations & iter.\ & oper.\ \\
		\hline
		rcv1    & $0.001$ & $7.06 \cdot 10^{8~}$ & $2.24 \cdot 10^{10}$ & $7.50 \cdot 10^{7~}$ & $4.63 \cdot 10^{9~}$ &  9.4 &  4.8 \\
		        &  $0.01$ & $9.21 \cdot 10^{7~}$ & $2.92 \cdot 10^{9~}$ & $1.86 \cdot 10^{7~}$ & $1.36 \cdot 10^{9~}$ &  5.0 &  2.1 \\
		        &   $0.1$ & $4.95 \cdot 10^{7~}$ & $1.57 \cdot 10^{9~}$ & $4.14 \cdot 10^{6~}$ & $4.43 \cdot 10^{8~}$ & 12.0 &  3.5 \\
		        &     $1$ & $2.36 \cdot 10^{7~}$ & $7.48 \cdot 10^{8~}$ & $1.53 \cdot 10^{6~}$ & $2.24 \cdot 10^{8~}$ & 15.4 &  3.3 \\
		        &    $10$ & $5.38 \cdot 10^{6~}$ & $1.71 \cdot 10^{8~}$ & $1.21 \cdot 10^{6~}$ & $1.79 \cdot 10^{8~}$ &  4.4 &  1.0 \\
		        &   $100$ & $4.25 \cdot 10^{5~}$ & $1.35 \cdot 10^{7~}$ & $2.36 \cdot 10^{5~}$ & $8.20 \cdot 10^{6~}$ &  1.8 &  1.6 \\
		\hline
		news~20 &   $0.1$ & $2.64 \cdot 10^{9~}$ & $1.78 \cdot 10^{10}$ & $3.88 \cdot 10^{7~}$ & $1.49 \cdot 10^{9~}$ & 68.0 & 11.9 \\
		        &     $1$ & $1.47 \cdot 10^{9~}$ & $9.89 \cdot 10^{9~}$ & $3.19 \cdot 10^{7~}$ & $7.50 \cdot 10^{8~}$ & 46.1 & 13.2 \\
		        &    $10$ & $3.78 \cdot 10^{8~}$ & $2.54 \cdot 10^{9~}$ & $2.30 \cdot 10^{7~}$ & $1.98 \cdot 10^{8~}$ & 16.4 & 12.8 \\
		        &   $100$ & $6.78 \cdot 10^{6~}$ & $4.55 \cdot 10^{7~}$ & $9.49 \cdot 10^{6~}$ & $6.42 \cdot 10^{7~}$ &  0.7 &  0.7 \\
		\hline
		E2006-tfidf & $0.001$ & $2.38 \cdot 10^{9~}$ & $3.16 \cdot 10^{11}$ & $4.08 \cdot 10^{7~}$ & $2.57 \cdot 10^{10}$ & 58.3 &  12.3 \\
		            &  $0.01$ & $3.40 \cdot 10^{8~}$ & $4.51 \cdot 10^{10}$ & $8.37 \cdot 10^{6~}$ & $4.02 \cdot 10^{8~}$ & 40.6 & 112.2 \\
		            &   $0.1$ & $2.59 \cdot 10^{7~}$ & $3.44 \cdot 10^{9~}$ & $5.70 \cdot 10^{6~}$ & $1.38 \cdot 10^{9~}$ &  4.5 &   2.5 \\
		            &     $1$ & $2.56 \cdot 10^{6~}$ & $3.40 \cdot 10^{8~}$ & $2.71 \cdot 10^{6~}$ & $3.75 \cdot 10^{8~}$ &  0.9 &   0.9 \\
	\end{tabular}
	\vspace{1em}
	\caption{ \label{tab:lasso-results}
		Performance of uniform CD (baseline) and the ACF-CD algorithm
		for LASSO training. The table lists numbers of iterations and
		operations, as well as the ``speed-up'' factor by which ACF-CD
		outperforms the uniform baseline (higher is better, values
		larger than one are speed-ups). The regularization parameter
		$\lambda$ so as to give the full range in between extremely
		sparse models with less than $10$ non-zeros and quite rich
		models with up to $10^4$ non-zero coefficients.
	}
\end{center}
\end{table}

The ACF-CD algorithm is never significantly slower than uniform CD and
in some cases faster by one to two orders of magnitude, while obtaining
solutions of equal quality (as indicated by the objective function
value). This marks a significant speed-up of ACF-CD over uniform CD.

\subsection{Linear SVM Training}

We compared ACF-CD in an extensive experimental study to the liblinear
SVM solver \cite{liblinear,hsieh:2008}. This is an extremely strong
and widely used baseline. The liblinear CD solver sweeps over random
permutations of coordinates in epochs. In addition it applies a
shrinking heuristic that removes bounded variables from the problem.
In other words the solver performs a simple type of online adaptation of
coordinate frequencies that is closely tied to the structure of the SVM
optimization problem, while ACF-CD applies its general-purpose
adaptation rule.

Our evaluation was based on six data sets listed in
table~\ref{table:data-linearsvm}. They range from medium sized to
extremely large.

\begin{table}[t]
\begin{center}
	\setlength{\tabcolsep}{0.5em}
	\begin{tabular}{l|r|r}
		Problem & Instances $(n = \ell)$ & Features $(d)$ \\
		\hline
		cover~type    & $581,012$ & $54$ \\
		kkd-a         & $8,407,752$ & $20,216,830$ \\
		kkd-b         & $19,264,097$ & $29,890,095$ \\
		news~20       & $19,996$ & $1,355,191$ \\
		rcv1          & $20,242$ & $47,236$ \\
		url           & $2,396,130$ & $3,231,961$ \\
	\end{tabular}
	\vspace{1em}
	\caption{ \label{table:data-linearsvm}
		Benchmark problems for linear SVM training.
	}
\end{center}
\end{table}

Both algorithms return accurate solutions to the SVM training problem.
The test errors coincide exactly. The algorithms don't differ in the
quality of the solution (dual objective values are extremely close;
often they coincide to 10 significant digits), but only in the time it
takes to compute this solution.
Training times of both algorithms are comparable since we have
implemented ACF-CD directly into the liblinear code.%
\footnote{An arbitrary outer loop iteration limit of $1000$ is
  hard-coded into liblinear version 1.9.2. We have removed this limit
  for the sake of a meaningful comparison.}

The results are reported compactly in figure~\ref{fig:svm-results}. The
figure includes three-fold cross validation performance which gives an
indication of which $C$ values are most relevant. The best value is
contained in the interior of the tested range in all cases. For
completeness, all timings and iteration numbers are listed tables
\ref{table:SVM-low} and~\ref{table:SVM-high}.

\begin{figure}[h!]
	\includegraphics[width=0.49\textwidth]{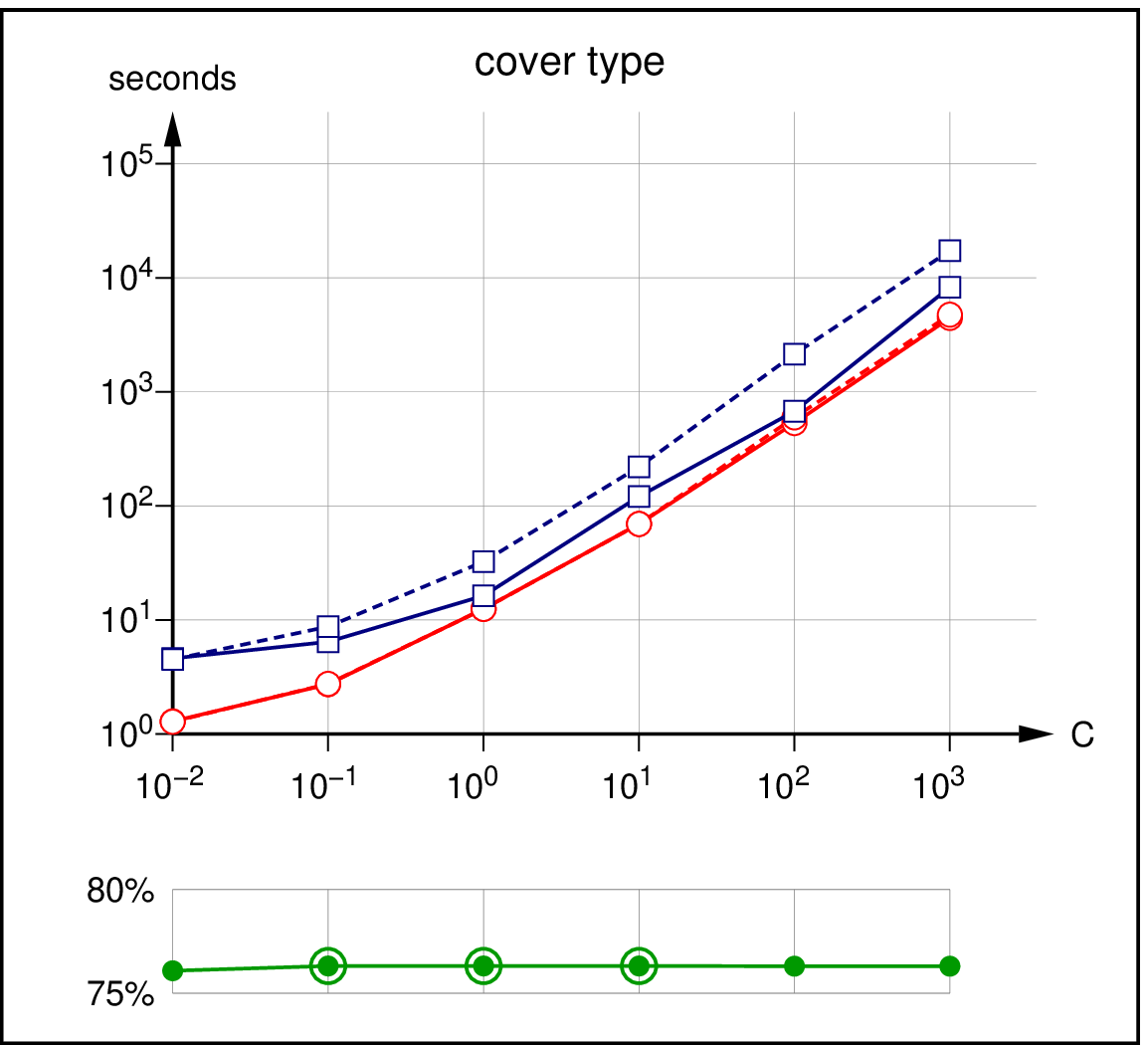}
	~~
	\includegraphics[width=0.49\textwidth]{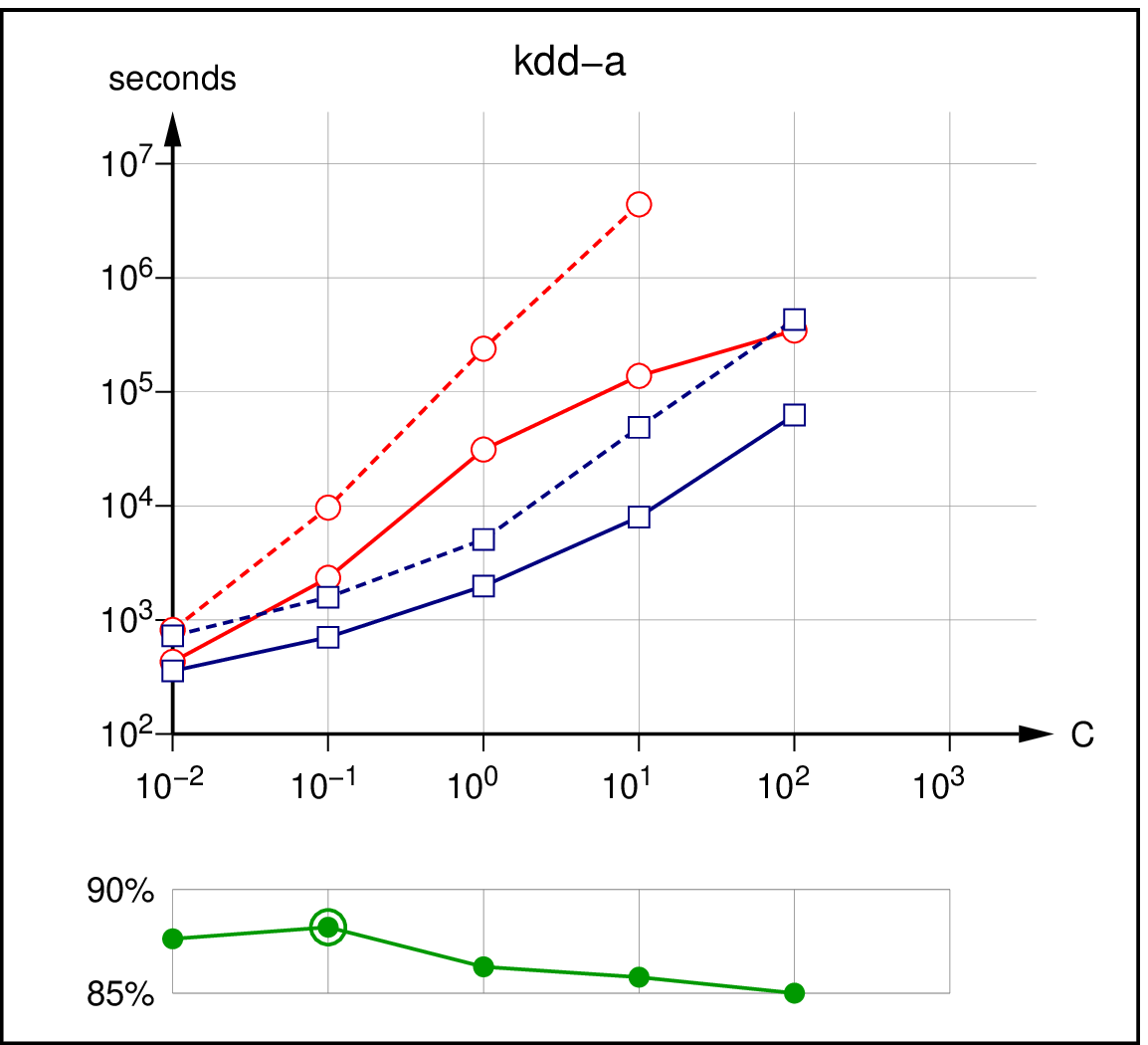}
	\\[1mm]
	\includegraphics[width=0.49\textwidth]{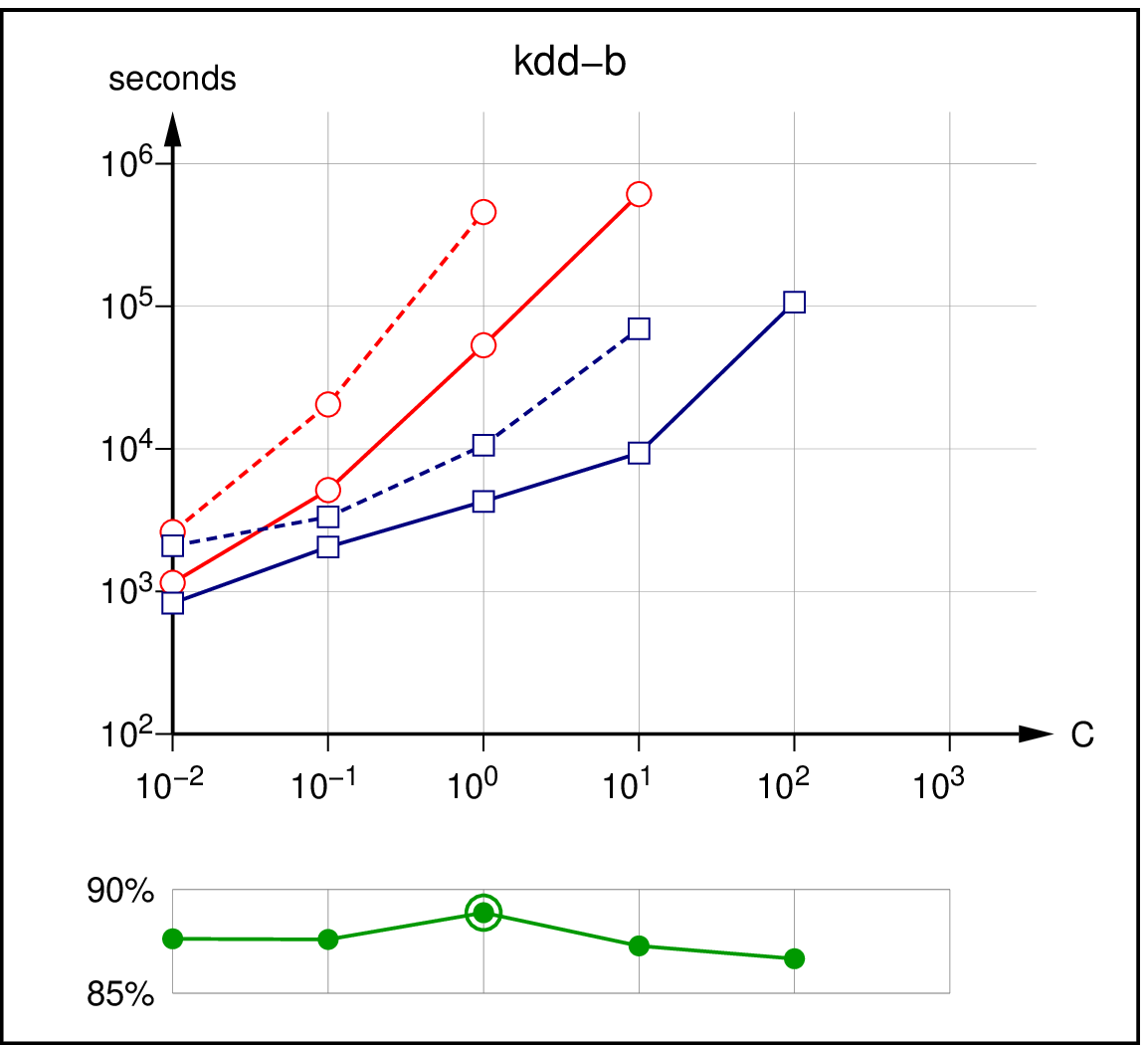}
	~~
	\includegraphics[width=0.49\textwidth]{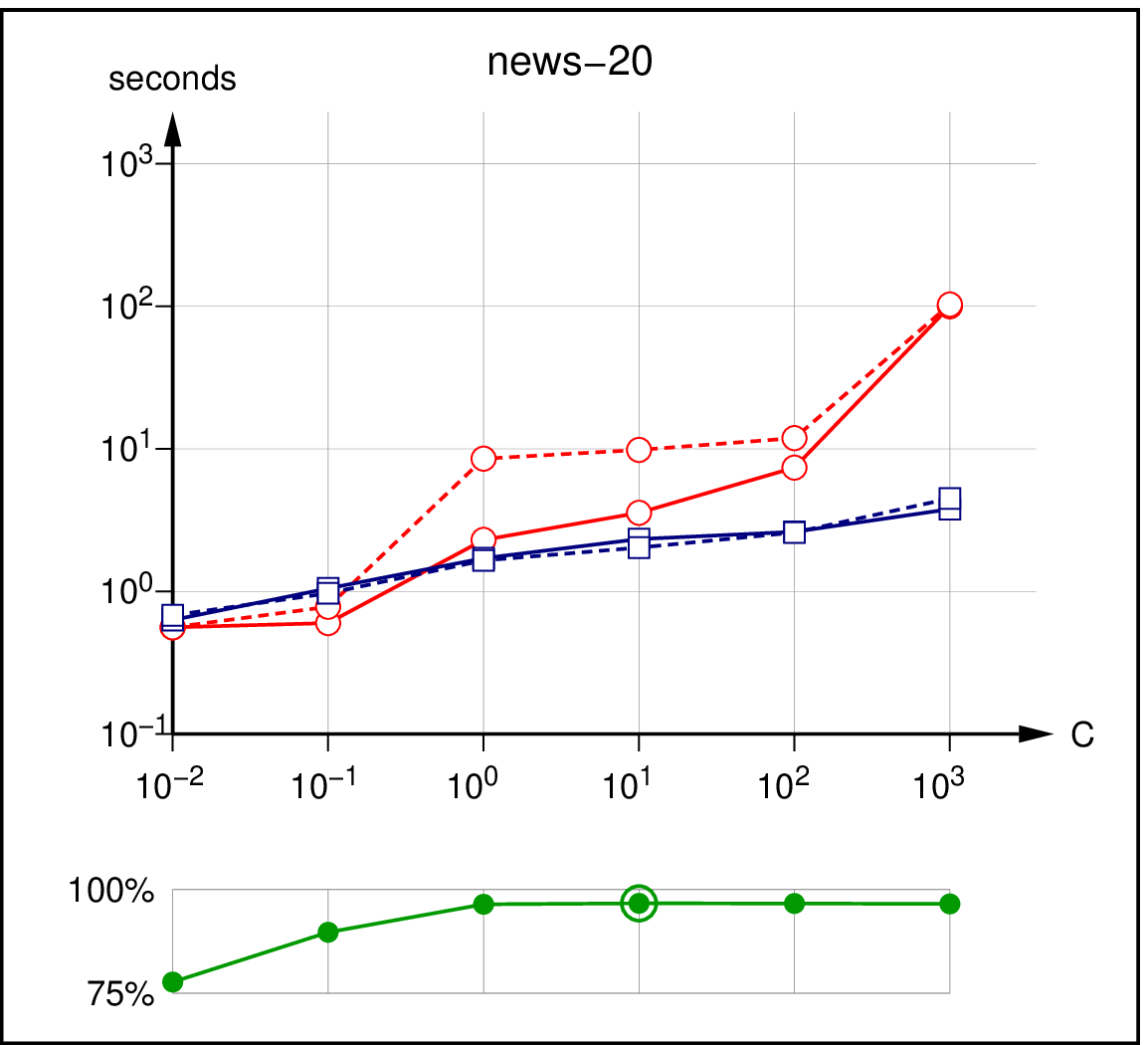}
	\\[1mm]
	\includegraphics[width=0.49\textwidth]{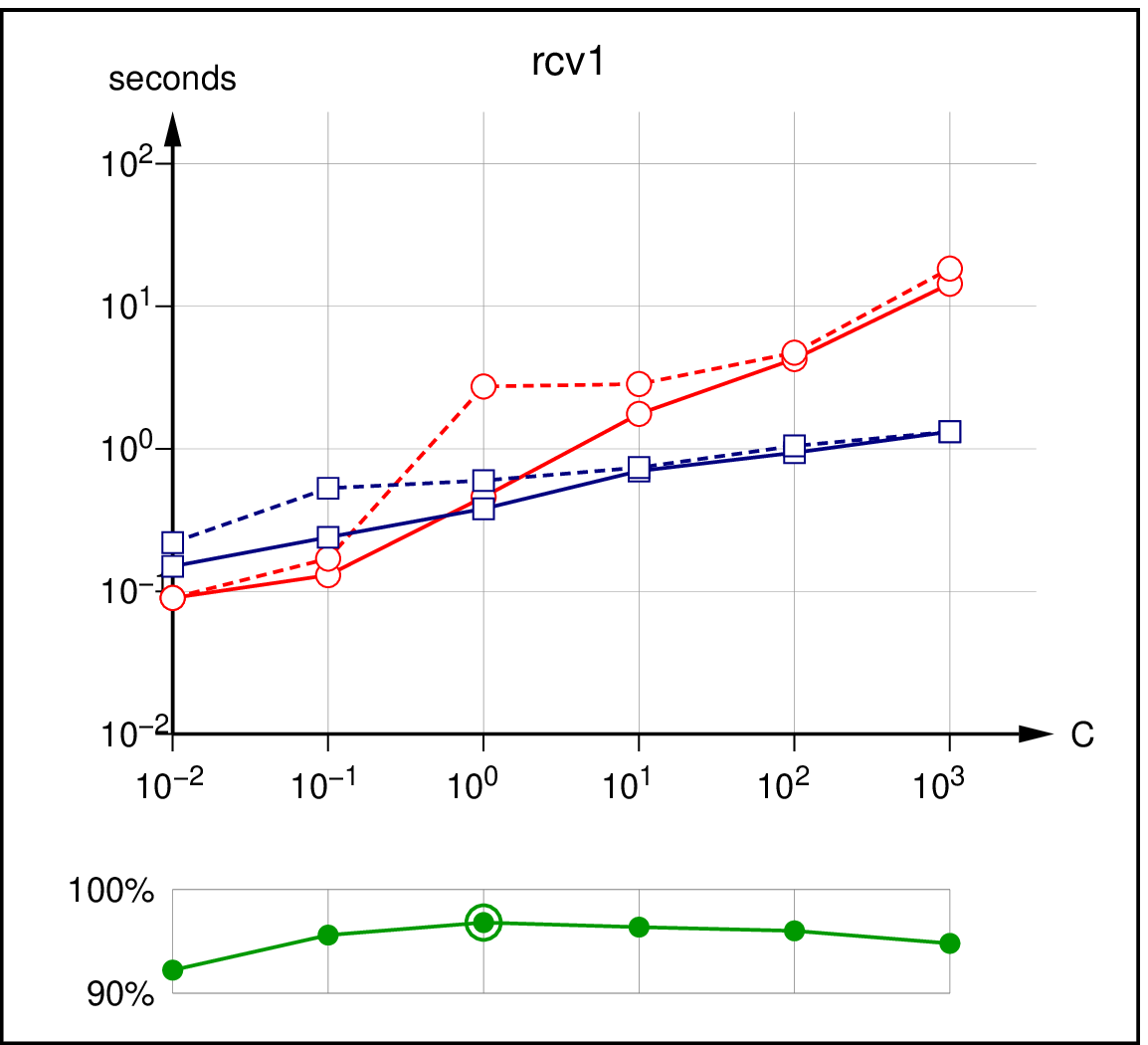}
	~~
	\includegraphics[width=0.49\textwidth]{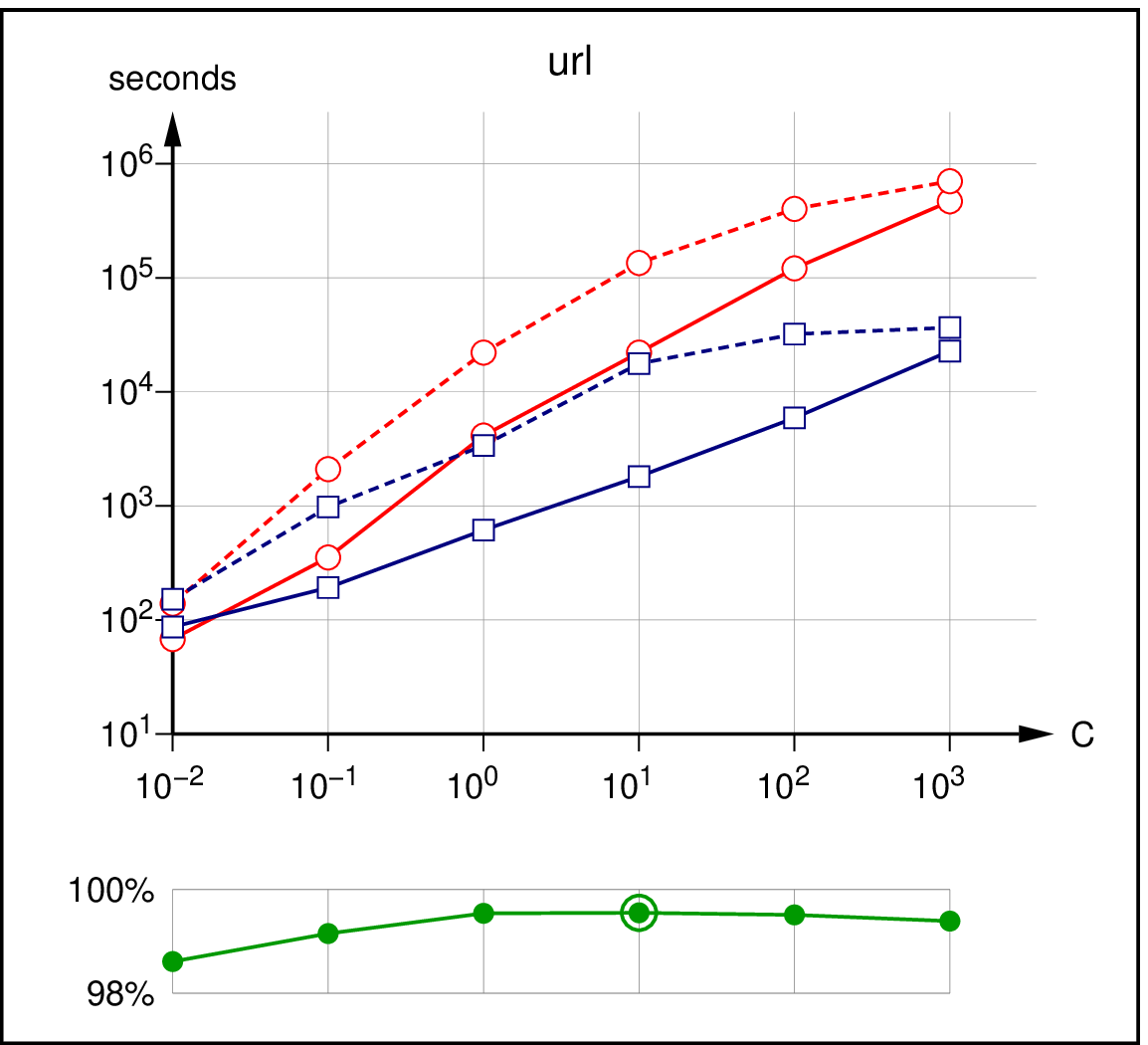}
\begin{center}
	\caption{ \label{fig:svm-results}
		Training times with the original liblinear algorithm
		(red circles) and with ACF-CD (blue squares) as a function of
		the regularization parameter~$C$. The target accuracy is
		$\varepsilon = 0.01$ for the solid curve and $\varepsilon = 0.001$
		for the dashed curves.
		For reference, three-fold cross validation performance (percent
		correct) is plotted below the curves in green, with best
		configurations circled. In all cases the best value(s) are
		contained in the interior of the chosen parameter range.
	}
\end{center}
\end{figure}

In most cases the ACF-CD algorithm is faster than liblinear. For large
values of $C$ it can outperform the baseline by more than an order of
magnitude (note the logarithmic scale in figure~\ref{fig:svm-results}).

The cover type problem is an exception. This problem is special for its
low feature dimensionality of only $54$ features, which means that the
$581,012$ dual variables are highly redundant. This implies that optimal
solution can be represented with many possible subsets of variables
$\alpha_i$ which makes adaptation of coordinate frequencies superfluous.
In this case the overhead of coordinate adaptation causes a considerable
slowdown.
It is actually well known that this problem can be solved more
efficiently in the primal, e.g., with liblinear's trust region method.

Overall (e.g., summing over all experiments) the ACF-CD method clearly
outperforms the liblinear algorith, which is a strong baseline.

\begin{table}[h]
\begin{center}
	\small
	\setlength{\tabcolsep}{0.25em}
	\begin{tabular}{l|c|rrrrrr}
		Data Set & Solver & $C=0.01$ & $ C=0.1$ & $C=1$ & $C=10$ & $C=100$ & $C=1000$ \\
		\hline
		\hline
		cover~type    & liblinear & $1.29$ & $2.73$ & $12.5$ & $69.5$ & $533$ & $4,450$ \\
		              &           & \TT{3.31 \cdot 10^{6~}} & \TT{7.41 \cdot 10^{6~}} & \TT{3.38 \cdot 10^{7~}} & \TT{1.80 \cdot 10^{8~}} & \TT{1.37 \cdot 10^{9~}} & \TT{1.14 \cdot 10^{10}} \\
		              & ACF       & $4.50$ & $6.43$ & $16.3$ & $121$ & $676$ & $8280$ \\
		              &           & \TT{8.92 \cdot 10^{6~}} & \TT{1.29 \cdot 10^{7~}} & \TT{3.31 \cdot 10^{7~}} & \TT{1.92 \cdot 10^{8~}} & \TT{1.49 \cdot 10^{9~}} & \TT{1.41 \cdot 10^{10}} \\
		\hline
		kkd-a         & liblinear & $429$ & $2,340$ & $31,200$ & $138,000$ & $345,000$ & --- \\
		              &           & \TT{3.07 \cdot 10^{8~}} & \TT{1.57 \cdot 10^{9~}} & \TT{1.88 \cdot 10^{10}} & \TT{8.77 \cdot 10^{10}} & \TT{2.35 \cdot 10^{11}} &        \\
		              & ACF       & $357$ &   $705$ & $1,980$ & $7,990$ & $62,700$ & --- \\
		              &           & \TT{2.93 \cdot 10^{8~}} & \TT{5.56 \cdot 10^{8~}} & \TT{1.49 \cdot 10^{9~}} & \TT{6.55 \cdot 10^{9~}} & \TT{5.21 \cdot 10^{10}} &        \\
		\hline
		kkd-b         & liblinear & $1,150$ & $5,140$ & $53,300$ & $612,000$ & --- & --- \\
		              &           & \TT{6.92 \cdot 10^{8~}} & \TT{2.86 \cdot 10^{9~}} & \TT{3.11 \cdot 10^{10}} & \TT{3.42 \cdot 10^{11}} &        &        \\
		              & ACF       &   $828$ & $2,050$ & $4,280$ &  $9,350$ & $107,000$ & --- \\
		              &           & \TT{6.50 \cdot 10^{8~}} & \TT{1.11 \cdot 10^{9~}} & \TT{2.91 \cdot 10^{9~}} & \TT{7.13 \cdot 10^{9~}} & \TT{8.04 \cdot 10^{10}} &        \\
		\hline
		news~20       & liblinear & $0.56$ & $0.60$ & $2.30$ & $3.56$ & $7.39$ & $100$ \\
		              &           & \TT{8.03 \cdot 10^{4~}} & \TT{1.22 \cdot 10^{5~}} & \TT{4.04 \cdot 10^{5~}} & \TT{6.38 \cdot 10^{5~}} & \TT{1.38 \cdot 10^{6~}} & \TT{2.47 \cdot 10^{7~}} \\
		              & ACF       & $0.63$ & $1.05$ & $1.71$ & $2.33$ & $2.62$ & $3.78$ \\
		              &           & \TT{1.20 \cdot 10^{5~}} & \TT{2.03 \cdot 10^{5~}} & \TT{3.37 \cdot 10^{5~}} & \TT{3.82 \cdot 10^{5~}} & \TT{4.82 \cdot 10^{5~}} & \TT{7.37 \cdot 10^{5~}} \\
		\hline
		rcv1          & liblinear & $0.09$ & $0.13$ & $0.46$ & $1.76$ & $4.27$ & $14.1$ \\
		              &           & \TT{9.36 \cdot 10^{4~}} & \TT{1.46 \cdot 10^{5~}} & \TT{4.77 \cdot 10^{5~}} & \TT{1.70 \cdot 10^{6~}} & \TT{4.19 \cdot 10^{6~}} & \TT{1.43 \cdot 10^{7~}} \\
		              & ACF       & $0.15$ & $0.24$ & $0.38$ & $0.70$ & $0.94$ & $1.32$ \\
		              &           & \TT{1.62 \cdot 10^{5~}} & \TT{2.78 \cdot 10^{5~}} & \TT{4.62 \cdot 10^{5~}} & \TT{7.98 \cdot 10^{5~}} & \TT{1.05 \cdot 10^{6~}} & \TT{1.51 \cdot 10^{6~}} \\
		\hline
		url           & liblinear & $67.9$ & $353$ & $4,140$ & $22,100$ & $121,000$ & $469,000$ \\
		              &           & \TT{4.05 \cdot 10^{7~}} & \TT{1.93 \cdot 10^{8~}} & \TT{2.22 \cdot 10^{9~}} & \TT{1.45 \cdot 10^{10}} & \TT{8.04 \cdot 10^{10}} & \TT{2.74 \cdot 10^{11}} \\
		              & ACF       & $86.7$ & $192$ & $614$ & $1,810$ & $5,910$ & $22,800$ \\
		              &           & \TT{6.24 \cdot 10^{7~}} & \TT{1.30 \cdot 10^{8~}} & \TT{4.24 \cdot 10^{8~}} & \TT{1.16 \cdot 10^{9~}} & \TT{4.34 \cdot 10^{9~}} & \TT{1.73 \cdot 10^{10}} \\
		\hline
	\end{tabular}
	\vspace{1em}
	\caption{ \label{table:SVM-low}
		Results of linear SVM training with low accuracy $\varepsilon = 0.01$.
		The table lists runtime in seconds and (small font below) the number
		of CD iterations.
		Runs marked with ``---'' did not finish after several weeks of training.
	}
\end{center}
\end{table}

\begin{table}[h]
\begin{center}
	\small
	\setlength{\tabcolsep}{0.25em}
	\begin{tabular}{l|c|rrrrrr}
		Problem & Solver & $C=0.01$ & $ C=0.1$ & $C=1$ & $C=10$ & $C=100$ & $C=1000$ \\
		\hline
		\hline
		cover~type    & liblinear & $1.28$ & $2.75$ & $12.5$ & $69.5$ & $597$ & $4,750$ \\
		              &           & \TT{3.31 \cdot 10^{6~}} & \TT{7.41 \cdot 10^{6~}} & \TT{3.38 \cdot 10^{7~}} & \TT{1.80 \cdot 10^{8~}} & \TT{1.78 \cdot 10^{9~}} & \TT{1.44 \cdot 10^{10}} \\
		              & ACF       & $4.50$ & $8.74$ & $32.4$ & $220$ & $2.140$ & $17,300$ \\
		              &           & \TT{8.92 \cdot 10^{6~}} & \TT{1.80 \cdot 10^{7~}} & \TT{6.45 \cdot 10^{7~}} & \TT{4.62 \cdot 10^{8~}} & \TT{4.32 \cdot 10^{9~}} & \TT{3.71 \cdot 10^{10}} \\
		\hline
		kkd-a         & liblinear &      $817$ &    $9,660$ &  $239,000$ & $4,410,000$ & --- & --- \\
		              &           & \TT{1.11 \cdot 10^{9~}} & \TT{9.16 \cdot 10^{9~}} & \TT{1.59 \cdot 10^{11}} & \TT{1.66 \cdot 10^{12}} &                         &                         \\
		              & ACF       &      $725$ &    $1,580$ &    $5,080$ &   $48,800$ &  $430,000$ & --- \\
		              &           & \TT{4.99 \cdot 10^{8~}} & \TT{8.90 \cdot 10^{8~}} & \TT{4.00 \cdot 10^{9~}} & \TT{3.32 \cdot 10^{10}} & \TT{2.67 \cdot 10^{11}} &                         \\
		\hline
		kdd-b         & liblinear &    $2,610$ &   $20,500$ &  $459,000$ & --- & --- & --- \\
		              &           & \TT{1.94 \cdot 10^{9~}} & \TT{1.17 \cdot 10^{10}} & \TT{2.73 \cdot 10^{11}} &                         &                         &                         \\
		              & ACF       &    $2,090$ &    $3,330$ &   $10,600$ &   $69,500$ &  --- & --- \\
		              &           & \TT{1.05 \cdot 10^{9~}} & \TT{1.77 \cdot 10^{9~}} & \TT{6.99 \cdot 10^{9~}} & \TT{4.28 \cdot 10^{10}} &                         &                         \\
		\hline
		news~20       & liblinear &     $0.56$ &     $0.78$ &     $8.54$ &     $9.84$ &     $11.9$ &      $103$ \\
		              &           & \TT{8.03 \cdot 10^{4~}} & \TT{1.54 \cdot 10^{5~}} & \TT{1.55 \cdot 10^{6~}} & \TT{1.87 \cdot 10^{6~}} & \TT{2.90 \cdot 10^{6~}} & \TT{2.50 \cdot 10^{7~}} \\
		              & ACF       &     $0.68$ &     $0.97$ &     $1.65$ &     $2.03$ &     $2.59$ &     $4.49$ \\
		              &           & \TT{1.20 \cdot 10^{5~}} & \TT{2.03 \cdot 10^{5~}} & \TT{3.37 \cdot 10^{5~}} & \TT{3.82 \cdot 10^{5~}} & \TT{4.82 \cdot 10^{5~}} & \TT{8.80 \cdot 10^{5~}} \\
		\hline
		rcv1          & liblinear &     $0.09$ &     $0.17$ &     $2.74$ &     $2.85$ &     $4.73$ &     $18.4$ \\
		              &           & \TT{9.40 \cdot 10^{4~}} & \TT{1.93 \cdot 10^{5~}} & \TT{3.36 \cdot 10^{6~}} & \TT{3.36 \cdot 10^{6~}} & \TT{5.63 \cdot 10^{6~}} & \TT{2.14 \cdot 10^{7~}} \\
		              & ACF       &     $0.22$ &     $0.53$ &     $0.60$ &     $0.74$ &     $1.05$ &     $1.32$ \\
		              &           & \TT{2.64 \cdot 10^{5~}} & \TT{6.09 \cdot 10^{5~}} & \TT{7.33 \cdot 10^{5~}} & \TT{9.14 \cdot 10^{5~}} & \TT{1.30 \cdot 10^{6~}} & \TT{1.70 \cdot 10^{6~}} \\
		\hline
		url           & liblinear &      $139$ &    $2,100$ &   $22,100$ &  $135,000$ &  $402,000$ &  $703,000$ \\
		              &           & \TT{8.27 \cdot 10^{7~}} & \TT{1.18 \cdot 10^{9~}} & \TT{1.46 \cdot 10^{10}} & \TT{7.61 \cdot 10^{10}} & \TT{2.35 \cdot 10^{11}} & \TT{3.78 \cdot 10^{11}} \\
		              & ACF       &      $152$ &     $978$ &    $3,390$ &   $17,700$ &   $32,100$ &   $36,600$ \\
		              &           & \TT{9.66 \cdot 10^{7~}} & \TT{5.92 \cdot 10^{8~}} & \TT{2.24 \cdot 10^{9~}} & \TT{9.76 \cdot 10^{9~}} & \TT{2.34 \cdot 10^{10}} & \TT{2.25 \cdot 10^{10}} \\
		\hline
	\end{tabular}
	\vspace{1em}
	\caption{ \label{table:SVM-high}
		Results of linear SVM training with high accuracy $\varepsilon = 0.001$.
		The table lists runtime in seconds and (small font below) the number
		of CD iterations.
		Runs marked with ``---'' did not finish after several weeks of training.
	}
\end{center}
\end{table}

\subsection{Multi-class SVM Training with Subspace Descent}

We evaluate a learning problem that naturally corresponds to a subspace
descent optimization problem, namely multi-class SVM training. The WW
multi-class SVM extension was implemented into the Shark~\cite{shark:2008}
machine learning library, version 3.0 (beta). The $K$-dimensional
sub-problems were solved with up to $10 \cdot K$ iterations of an inner
CD solver picking the largest derivative component for descent (here $K$
denotes the number of classes).

The data sets for evaluation are listed in table~\ref{table:data-mc}.
A separate test set was used to estimate a reasonable range for the
regularization parameter~$C$. The parameter was varied on a grid of the
form $C = 10^k$, and result are reported for a grid of size~$5$ around
the best value.

As discussed above, the liblinear algorithm applies a shrinking technique
to reduce the problem size during the optimization run. This technique
does not carry over in a one-to-one fashion to the multi-class problem.
For comparison we have implemented a similar shrinking heuristic into
the multi-class SVM solver. However, for this problem the heuristic did
not perform significantly better than the uniform baseline and sometimes
lead to considerably longer optimization times due to wrong shrinking
decisions. This is different from the binary SVM case where shrinking
works well in most cases. Therefore we have dropped shrinking and
instead compare ACF-CD against the better performing uniform baseline.

\begin{table}[h]
\begin{center}
	\setlength{\tabcolsep}{0.5em}
	\begin{tabular}{l|c|c|c}
		Problem & Instances $(n = \ell)$ & Features $(d)$ & Classes $(K)$ \\
		\hline
		iris     &    105 &      4 &  3 \\
		soybean  &    214 &     35 & 19 \\
		news-20  & 15,935 & 62,061 & 20 \\
		rcv1     & 15,564 & 47,236 & 53
	\end{tabular}
	\vspace{1em}
	\caption{ \label{table:data-mc}
		Benchmark problems for multi-class SVM
		(subspace descent) experiments.
	}
\end{center}
\end{table}

The experimental results are presented in table~\ref{table:results-WW}.
The ACF algorithm clearly outperforms the uniform coordinate selection
baseline. It is noteworthy that ACF does not only perform better but
also scales much more gracefully to hard optimization problems,
corresponding to large values of~$C$.

\begin{table}[h]
\begin{center}
	\setlength{\tabcolsep}{0.5em}
	\begin{tabular}{l|l|r|r|c|r|c|r|r}
		                 &     & test     & \multicolumn{2}{|c}{\textbf{uniform}} & \multicolumn{2}{|c|}{\textbf{ACF}} & \multicolumn{2}{|c}{\textbf{speed-up}} \\
		\textbf{problem} & $C$ & accuracy & iterations & seconds                  & iterations & seconds               & iter.\ & time                           \\
		\hline
		     & $10^{-2}$ &  60.0\% &     4,095 & 0.003 &   2,625 & 0.002 &  1.6 &  1.5 \\
		     & $10^{-1}$ &  60.0\% &    42,735 & 0.023 &  11,130 & 0.007 &  3.8 &  3.3 \\
		iris & $10^0$    & 100.0\% &   238,140 & 0.116 &  27,300 & 0.009 &  8.7 & 12.8 \\
		     & $10^1$    &  95.6\% & 5,007,870 & 2.16  & 410,445 & 0.279 & 12.2 &  7.7 \\
		     & $10^2$    &  95.6\% & 2,095,065 & 0.959 & 267,855 & 0.194 &  7.8 &  4.9 \\
		\hline
		        & $10^{-2}$ & 69.9\% &    20,972 & 0.0434 &    12,412 & 0.027 & 1.7 & 1.6 \\
		        & $10^{-1}$ & 88.2\% &    78,752 & 0.150  &    42,800 & 0.101 & 1.8 & 1.5 \\
		soybean & $10^0$    & 91.4\% &   377,282 & 0.664  &    93,732 & 0.218 & 4.0 & 3.0 \\
		        & $10^1$    & 86.0\% &   607,974 & 1.04   &   113,206 & 0.271 & 5.3 & 3.8 \\
		        & $10^2$    & 81.7\% & 7,038,032 & 11.8   & 1,346,916 & 2.79  & 5.2 & 4.2 \\
		\hline
		        & $10^{-4}$ & 76.7\% &    270,895 & 2.62 &   334,635 & 3.22 &  0.8 &  0.8 \\
		        & $10^{-3}$ & 81.9\% &  2,230,900 & 23.7 &   318,700 & 3.23 &  7.0 &  7.3 \\
		news~20 & $10^{-2}$ & 83.4\% &  1,290,735 & 14.4 &   462,115 & 5.04 &  2.8 &  2.9 \\
		        & $10^{-1}$ & 81.5\% &  6,023,430 & 63.3 &   780,815 & 8.77 &  7.7 &  7.2 \\
		        & $10^0$    & 79.2\% & 60,234,300 &  632 & 1,481,955 & 18.9 & 40.6 & 33.4 \\
		\hline
		     & $10^{-2}$ & 81.6\% &    513,612 & 15.5 &   295,716 & 9.13 &  1.7 &  1.7 \\
		     & $10^{-1}$ & 87.8\% &  2,241,216 & 76.4 &   513,612 & 18.2 &  4.4 &  4.1 \\
		rcv1 & $10^0$    & 88.8\% &  4,264,536 &  153 &   606,996 & 23.1 &  7.0 &  6.6 \\
		     & $10^1$    & 88.2\% & 12,746,916 &  468 &   793,764 & 32.6 & 16.0 & 14.3 \\
		     & $10^2$    & 87.8\% & 10,287,804 &  381 &   996,096 & 40.0 & 10.3 &  9.5 \\
	\end{tabular}
	\vspace{1em}
	\caption{ \label{table:results-WW}
		Number of iterations and training times in seconds for
		multi-class SVM training with subspace descent, as well as
		corresponding speed-up factors (higher is better). Test errors
		indicate that the parameter $C$ is varied within a reasonable
		range.
	}
\end{center}
\end{table}

\subsection{Logistic Regression}

We have implemented ACF-CD into the liblinear logistic regression solver
~\cite{liblinear,yu:2011}, analog to the linear SVM solver. There are
two major differences to the linear SVM case. First, the dual logistic
regression solution is not sparse and thus shrinking is not applicable.
Hence, liblinear applies uniform coordinate selection. Second, the
one-dimensional sub-problems cannot be solved analytically. Instead a
series of Newton steps is applied.

\begin{table}[h]
\begin{center}
	\setlength{\tabcolsep}{0.5em}
	\begin{tabular}{l|l|r|r|c|r|c|r|r}
		                 &     & 3-fold & \multicolumn{2}{|c}{\textbf{liblinear}} & \multicolumn{2}{|c|}{\textbf{ACF}} & \multicolumn{2}{|c}{\textbf{speed-up}} \\
		\textbf{problem} & $C$ & CV     & iterations & seconds & iterations & seconds & iter.\ & time \\
		\hline
		        & $10^2$ & 96.3\% &     1,459,708 &   4.22 &     660,500 &  1.84 &   2.2 &   2.3 \\
		        & $10^3$ & 96.4\% &    14,317,136 &   41.7 &     560,268 &  1.77 &  25.6 &  23.6 \\
		news~20 & $10^4$ & 96.5\% &   136,972,600 &    406 &   2,019,200 &  5.54 &  67.8 &  73.3 \\
		        & $10^5$ & 96.4\% & 1,292,781,392 &  3,720 &  10,355,564 &  27.7 & 124.8 & 134.3 \\
		        & $10^6$ & 96.5\% & 3,929,561,008 & 26,834 & 158,291,120 &   382 &  24.8 &  70.3 \\
		\hline
		     & $10^0$ & 96.1\% &     121,452 & 0.081 &    545,462 & 0.252 &  0.2 &  0.3 \\
		     & $10^1$ & 96.7\% &     263,146 & 0.153 &    345,099 & 0.188 &  0.8 &  0.8 \\
		rcv1 & $10^2$ & 96.7\% &   1,497,908 & 0.819 &    788,883 & 0.341 &  1.9 &  2.4 \\
		     & $10^3$ & 96.5\% &  14,007,464 &  7.46 &  2,389,345 &  1.03 &  5.9 &  7.2 \\
		     & $10^4$ & 96.3\% & 136,005,998 &  72.9 & 11,274,174 &  4.85 & 12.1 & 15.0 \\
		\hline
		    & $10^0$ & 99.17\% &   251,593,650 &    565 &    217,914,274 &     445 & 1.2 &  1.3 \\
		    & $10^1$ & 99.42\% & 2,317,057,710 &  4,265 &    723,665,786 &   1,478 & 3.2 &  2.9 \\
		url & $10^2$ & 99.44\% &   665,404,720 & 51,734 &  1,940,658,131 &   3,992 & 0.3 & 13.0 \\
		    & $10^3$ & 99.42\% & ---           &    --- & 16,933,597,906 &  44,115 & --- & --- \\
		    & $10^4$ & ---     & ---           &    --- & 56,218,463,548 & 139,041 & --- & --- \\
	\end{tabular}
	\vspace{1em}
	\caption{ \label{table:results-logreg}
		Number of iterations and training times in seconds for logistic
		regression, as well as corresponding speed-up factors (higher is
		better). Three-fold cross-validation performance indicates that
		the parameter $C$ is varied within a reasonable range.
		The \emph{url} runs for $C \geq 1,000$ with liblinear are marked
		with ``---''. They had to be stopped without finishing after five
		days of training.
	}
\end{center}
\end{table}

The data sets news~20, rcv1, and url from table~\ref{table:data-linearsvm}
were used for comparison. We have tuned the regularization parameter~$C$
on the grid $10^k$ based on three-fold cross-validation~(CV).
In table~\ref{table:results-logreg} we report results for a problem
specific range of five settings centered on the best three-fold CV
performance.

The results exhibit the usual pattern. In some highly regularized cases
ACF-CD is a bit slower, but in most cases it improves performance.
Saving are most significant where they are most relevant, namely for
parameter configurations that result in long training times. On the
logistic regression problem ACF-CD is up to two orders of magnitude
faster than the liblinear solver.

\subsection{Discussion}

Our results show that the ACF algorithm is superior to uniform CD. Of
course the algorithm can unfold its full potential only if it performs
sufficiently many sweeps over the coordinates. Highly regularized
machine learning problems tend to be simple in the sense that the
stopping criterion can be met already after very few sweeps. In this
case uniform coordinate selection is usually a good strategy, and the
computational overhead of adaptation can be saved. This is why ACF does
not beat the uniform baseline in all cases, in particular for small
values of the regularization parameter~$C$. However, these optimization
runs are anyway extremely fast and do not pose a computational challenge.
As soon as the problem becomes more involved ACF starts to pay off,
often saving $90\%$ of the training time and sometimes even more. This
is the dominating effect, e.g., when performing grid search or a
parameter study. The ACF algorithm does not only outperform the uniform
selection baseline but also the standard SVM training algorithm with its
shrinking technique. The only exception is the cover type data set.
Shrinking is a strong competitor since it is a domain-specific technique
designed explicitly to take a-priori knowledge about the dual SVM
solution into account. In contrast, ACF is a generic CD speedup
technique applicable to all problem types. We argue that outperforming a
problem specific technique such as shrinking with a general purpose
method such as ACF is a strong result.

\section{Conclusion}

We have introduced the Adaptive Coordinate Frequencies (ACF) algorithm.
It adapts the relative frequencies of coordinates in coordinate descent
(CD) optimization \emph{online} to the problem at hand. The aim of the
adaptation is to maximize convergence speed and thus to minimize the
time complexity of various machine training problems.

This technique allows to efficiently solve CD problems that greatly
profit from non-uniform coordinate selection probabilities. The need
for non-unifor\-mity is obvious at least in a machine learning context
where coordinates correspond to data points or features, some of which
are known to be more important than others, only it is hard to say
beforehand which ones are how important. Our method allows to start the
CD algorithm with an uninformed guess---the uniform distribution---or a
more informed choice if available, and to adapt the coordinate selection
distribution online for optimal progress. This is particularly helpful
for problems with \emph{changing} importance of coordinates, e.g., when
constraints become active or inactive.

We have presented a first analysis of the ACF method based on a Markov
chain perspective. We conjecture that the coordinate selection
distribution that maximizes the convergence rate is characterized by
equal progress in each coordinate. This property, granted that it holds
true in sufficient generality, provides an explanation of why the ACF
algorithm works so well. We show that in expectation and under certain
simplifying assumptions the ACF algorithm drives the coordinate
selection distribution towards this equilibrium. Extending our
understanding of this process and the underlying Markov chains resulting
from coordinate descent is a primary research goal for future
investigations.

It turns out that many successful applications of CD algorithms in
machine learning rely on uniform coordinate selection. The only notable
exception is linear SVM training where a shrinking heuristic can set a
coordinate probability to exact zero. We compared ACF to state-of-the-art
machine training implementations for four different problems. Overall
the new algorithm shows impressive performance. It systematically
outperforms the established algorithm, sometimes by an order of
magnitude and more, and falls behind only in rare special cases. We
therefore recommend online adaptation of coordinate frequencies as a
general tool for coordinate descent optimization, in particular in the
domain of machine learning.

\bibliographystyle{plain}

\begin{thebibliography}{10}

\bibitem{ruelle:1979}
D.~Ruelle.
\newblock Ergodic theory of differential dynamical systems.
\newblock {\em Publications math{\'e}matiques de l'I.H.{\'E}.S.}, 50:27--58,
  1979.

\bibitem{IGO-TR}
L.~Arnold, A.~Auger, N.~Hansen, and Y.~Ollivier.
\newblock Information-geometric optimization algorithms: a unifying picture via
  invariance principles.
\newblock Technical Report arXiv:1106.3708, arxiv.org, 2011.

\bibitem{bottou:1998}
L.~Bottou.
\newblock Online algorithms and stochastic approximations.
\newblock In D.~Saad, editor, {\em Online Learning and Neural Networks}.
  Cambridge University Press, Cambridge, UK, 1998.

\bibitem{bredensteiner:1999}
E.~J. Bredensteiner and K.~P. Bennett.
\newblock Multicategory classification by support vector machines.
\newblock {\em Computational Optimization and Applications}, 12(1):53--79,
  1999.

\bibitem{crammer:2002}
K.~Crammer and Y.~Singer.
\newblock {On the algorithmic implementation of multiclass kernel-based vector
  machines}.
\newblock {\em Journal of Machine Learning Research}, 2:265--292, 2002.

\bibitem{liblinear}
R.~E. Fan, K.~W. Chang, C.~J. Hsieh, X.~R. Wang, and C.~J. Lin.
\newblock {LIBLINEAR}: A library for large linear classification.
\newblock {\em The Journal of Machine Learning Research}, 9:1871--1874, 2008.

\bibitem{fan:2005}
R.~E. Fan, P.~H. Chen, and C.~J. Lin.
\newblock {Working Set Selection Using Second Order Information for Training
  Support Vector Machines}.
\newblock {\em Journal of Machine Learning Research}, 6:1889--1918, 2005.

\bibitem{friedman:2007}
J.~Friedman, T.~Hastie, H.~H{\"o}fling, and R.~Tibshirani.
\newblock {Pathwise Coordinate Optimization}.
\newblock {\em The Annals of Applied Statistics}, 1(2):302--332, 2007.

\bibitem{glasmachers:2013acml}
T.~Glasmachers and {\"U}.~Dogan.
\newblock Accelerated coordinate descent with adaptive coordinate frequencies.
\newblock In {\em Proceedings of the fifth Asian Conference on Machine Learning
  (ACML)}, 2013.

\bibitem{glasmachers:2013arxiv}
T.~Glasmachers and {\"U}.~Dogan.
\newblock {Accelerated Linear SVM Training with Adaptive Variable Selection
  Frequencies}.
\newblock Technical Report arXiv:1302.5608, arxiv.org, 2013.

\bibitem{glasmachers:2006}
T.~Glasmachers and C.~Igel.
\newblock {Maximum-gain Working Set Selection for {SVMs}}.
\newblock {\em Journal of Machine Learning Research}, 7:1437--1466, 2006.

\bibitem{glasmachers:2010}
T.~Glasmachers, T.~Schaul, Y.~Sun, D.~Wierstra, and J.~Schmidhuber.
\newblock Exponential natural evolution strategies.
\newblock In {\em Proceedings of the Genetic and Evolutionary Computation
  Conference (GECCO)}, 2010.

\bibitem{hansen:2001}
N.~Hansen and A.~Ostermeier.
\newblock {Completely Derandomized Self-Adaptation in Evolution Strategies}.
\newblock {\em Evolutionary Computation}, 9(2):159--195, 2001.

\bibitem{hoffman:2013}
M.~D. Hoffman, D.~M. Blei, C.~Wang, and J.~Paisley.
\newblock Stochastic variational inference.
\newblock {\em Journal of Machine Learning Research}, 14:1303--1347, 2013.

\bibitem{hsieh:2008}
C.~J. Hsieh, K.~W. Chang, C.~J. Lin, S.~S. Keerthi, and S.~Sundararajan.
\newblock A dual coordinate descent method for large-scale linear {SVM}.
\newblock In {\em Proceedings of the 30th International Conference on Machine
  learning (ICML)}, volume 951, pages 408--415, 2008.

\bibitem{hsieh:2011}
C.-J. Hsieh and I.~S. Dhillon.
\newblock Fast coordinate descent methods with variable selection for
  non-negative matrix factorization.
\newblock In {\em Proceedings of the 17th ACM SIGKDD International Conference
  on Knowledge Discovery and Data Mining}, pages 1064--1072. ACM, 2011.

\bibitem{shark:2008}
C.~Igel, V.~Heidrich-Meisner, and T.~Glasmachers.
\newblock Shark.
\newblock {\em Journal of Machine Learning Research}, 9:993--996, 2008.

\bibitem{igel:2003}
C.~Igel and M.~H{\"u}sken.
\newblock {Empirical Evaluation of the Improved Rprop Learning Algorithm}.
\newblock {\em Neurocomputing}, 50:105--123, 2003.

\bibitem{joachims:1998}
T.~Joachims.
\newblock {Making Large-Scale {SVM} Learning Practical}.
\newblock In B.~Sch{\"o}lkopf, C.~Burges, and A.~Smola, editors, {\em Advances
  in Kernel Methods -- Support Vector Learning}, chapter~11, pages 169--184.
  MIT Press, 1998.

\bibitem{lee:2004}
Y.~Lee, Y.~Lin, and G.~Wahba.
\newblock Multicategory support vector machines: {T}heory and application to
  the classification of microarray data and satellite radiance data.
\newblock {\em Journal of the American Statistical Association},
  99(465):67--82, 2004.

\bibitem{lin:2001}
C.-J. Lin.
\newblock Linear convergence of a decomposition method for support vector
  machines.
\newblock Technical report, Department of Computer Science, National Taiwan
  University, 2001.

\bibitem{loshchilov:2011}
I.~Loshchilov, M.~Schoenauer, and M.~Sebag.
\newblock {Adaptive Coordinate Descent}.
\newblock In N.~Krasnogor, editor, {\em Genetic and Evolutionary Computation
  Conference (GECCO)}. ACM, 2011.

\bibitem{luo:1992}
Z.-Q. Luo and P.~Tseng.
\newblock On the convergence of the coordinate descent method for convex
  differentiable minimization.
\newblock {\em Journal of Optimization Theory and Applications}, 72(1):7--35,
  1992.

\bibitem{nesterov:2012}
Y.~Nesterov.
\newblock {Efficiency of Coordinate Descent Methods on Huge-Scale Optimization
  Problems}.
\newblock {\em SIAM Journal on Optimization}, 22(2):341--362, 2012.

\bibitem{platt:1998}
J.~C. Platt.
\newblock {Fast Training of Support Vector Machines using Sequential Minimal
  Optimization}.
\newblock In B.~Sch{\"o}lkopf, C.~Burges, and A.~Smola, editors, {\em Advances
  in Kernel Methods -- Support Vector Learning}, pages 185--208. MIT Press,
  1998.

\bibitem{rechenberg:1978}
I.~Rechenberg.
\newblock {\em Evolutionsstrategien}.
\newblock Springer, 1978.

\bibitem{richtarik:2012}
P.~Richt{\'a}rik and M.~Tak{\'a}{\v{c}}.
\newblock Iteration complexity of randomized block-coordinate descent methods
  for minimizing a composite function.
\newblock {\em Mathematical Programming}, pages 1--38, 2012.

\bibitem{richtarik:2013}
P.~Richt{\'a}rik and M.~Tak{\'a}{\v{c}}.
\newblock {On Optimal Probabilities in Stochastic Coordinate Descent Methods}.
\newblock Technical Report arXiv:1310.3438, arxiv.org, 2013.

\bibitem{riedmiller:1993}
M.~Riedmiller and H.~Braun.
\newblock {A direct adaptive method for faster backpropagation learning: The
  {RPROP} algorithm}.
\newblock In {\em Neural Networks, 1993., IEEE International Conference on},
  pages 586--591. IEEE, 1993.

\bibitem{schaul:2013}
T.~Schaul, S.~Zhang, and Y.~LeCun.
\newblock {No More Pesky Learning Rates}.
\newblock In {\em Proceedings of the 25th International Conference on Machine
  Learning (ICML)}, volume~28, pages 343--351, 2013.

\bibitem{shalev:2013}
S.~Shalev-Shwartz and T.~Zhang.
\newblock {Stochastic Dual Coordinate Ascent Methods for Regularized Loss
  Minimization}.
\newblock {\em The Journal of Machine Learning Research}, 14:567--599, 2013.

\bibitem{steinwart:2011}
I.~Steinwart, D.~Hush, and C.~Scovel.
\newblock {Training {SVMs} Without Offset}.
\newblock {\em The Journal of Machine Learning Research}, 12:141--202, 2011.

\bibitem{suttorp:2009}
T.~Suttorp, N.~Hansen, and C.~Igel.
\newblock {Efficient Covariance Matrix Update for Variable Metric Evolution
  Strategies}.
\newblock {\em Machine Learning}, 75(2):167--197, 2009.

\bibitem{tseng:2001}
P.~Tseng.
\newblock {Convergence of a Block Coordinate Descent Method for
  Nondifferentiable Minimization}.
\newblock {\em Journal of optimization theory and applications},
  109(3):475--494, 2001.

\bibitem{vapnik:1998}
V.~Vapnik.
\newblock {\em {Statistical Learning Theory}}.
\newblock John Wiley and Sons, 1998.

\bibitem{weston:1999}
J.~Weston and C.~Watkins.
\newblock Support vector machines for multi-class pattern recognition.
\newblock In M.~Verleysen, editor, {\em Proceedings of the Seventh European
  Symposium On Artificial Neural Networks (ESANN)}, pages 219--224. d-side
  publications, 1999.

\bibitem{yu:2011}
H.-F. Yu, F.-L. Huang, and C.-J. Lin.
\newblock Dual coordinate descent methods for logistic regression and maximum
  entropy models.
\newblock {\em Machine Learning}, 85:41--75, 2011.

\bibitem{yuan:2012a}
G.-X. Yuan, C.-H. Ho, and C.-J. Lin.
\newblock {An improved GLMNET for L1-regularized Logistic Regression}.
\newblock {\em The Journal of Machine Learning Research}, 13:1999--2030, 2012.

\end{thebibliography}

\end{document}